\documentclass[10pt,twocolumn,letterpaper]{article}

\usepackage{iccv}
\usepackage{times}
\usepackage{epsfig}
\usepackage{graphicx}
\usepackage{amsmath}
\usepackage{amssymb}
\usepackage{amsmath, amsthm, amssymb, amsfonts}
\usepackage{booktabs}
\usepackage{soul}
\usepackage[dvipsnames,table]{xcolor}
\usepackage{xcolor,colortbl}
\usepackage{multirow}
\usepackage{parskip}
\usepackage{soul}
\usepackage{makecell}
\usepackage{afterpage}
\usepackage{float}
\usepackage[accsupp]{axessibility} 

\usepackage[pagebackref=true,breaklinks=true,letterpaper=true,colorlinks,bookmarks=false]{hyperref}
\usepackage[capitalize]{cleveref}

\iccvfinalcopy 

\def\iccvPaperID{11330} 

\definecolor{dgreen}{rgb}{0.0, 0.5, 0.0}

\newcommand{\algo}{\mathcal{A}}

\newcommand\norm[1]{\left\lVert#1\right\rVert}
\DeclareMathOperator{\softmax}{softmax}
\DeclareMathOperator{\ca}{CA}

\def\name{SAFE\xspace}

\renewcommand{\paragraph}[1]{\noindent\textbf{#1}}

\newtheorem{definition}{Definition}
\newtheorem*{theorem*}{Theorem}
\newtheorem{theorem}{Theorem}

\ificcvfinal\fi

\begin{document}

\title{\name: Machine Unlearning With Shard Graphs}

\author{
Yonatan Dukler$^{1*}$,
Benjamin Bowman$^{1,2\dagger*}$,
Alessandro Achille$^{1*}$,
Aditya Golatkar$^{1,2\dagger}$, \\
Ashwin Swaminathan$^{1}$,
Stefano Soatto$^{1}$ \\
AWS AI Labs$^{1}$, UCLA$^{2}$\\
{\tt\{dukler,bowmaben,aachille,agolatka,swashwin,soattos\}@amazon.com}\\
}

\maketitle
\ificcvfinal\thispagestyle{empty}\fi

\begin{abstract}
We present Synergy Aware Forgetting Ensemble (SAFE), a method to adapt large models on a diverse collection of data while minimizing the expected cost to remove the influence of training samples from the trained model. This process, also known as selective forgetting or unlearning, is often conducted by partitioning a dataset into shards, training fully independent models on each, then ensembling the resulting models.  Increasing the number of shards reduces the expected cost to forget but at the same time it increases inference cost and reduces the final accuracy of the model since synergistic information between samples is lost during the independent model training. 
Rather than treating each shard as independent, \name introduces the notion of 
a shard graph, which allows incorporating limited information from other shards during training, trading off a modest increase in expected forgetting cost with a significant increase in accuracy, all while still attaining complete removal of residual influence after forgetting. \name uses a lightweight system of adapters which can be trained while reusing most of the computations. This allows \name to be trained on shards an order-of-magnitude smaller than current state-of-the-art methods (thus reducing the forgetting costs) while also maintaining high accuracy, as we demonstrate empirically on fine-grained computer vision datasets.
\end{abstract}

\section{Introduction}
Large-scale neural networks are typically trained on large monolithic datasets. However, real world data often comes from many different sources which may require different treatment depending on their terms of use. In some cases, the terms can change, triggering the need to not just erase a portion of the data, but also remove its influence on the trained model.
\footnotetext[1]{Equal contributions} 
\footnotetext[2]{Work done during an internship at AWS AI Labs.}
\hspace{-0.3cm} If the model is trained with the entire dataset in an undifferentiated fashion, even a request to remove a small fraction of the data may result in re-training the entire model on the complement.
Considering the scale of large neural networks currently in use, re-training the model after each data erasure is costly. Thus, it is beneficial to develop methods to trace the influence of various segments of data onto the trained model, and to remove their effect if needed, especially as the scale of production models and the amount of training data continues to grow.

One simple and robust approach to forgetting is compartmentalizing the information of different subsets of data into distinct model parameters. In this scenario, the training data is split into disjoint shards, and different parameters or ``adapters'' are trained separately on each shard, and then ensembled to obtain a final model. The advantage of this approach is that, if the influence of a sample needs to be removed, only the parameters corresponding to its shard have to be retrained. Moreover if an entire source of data needs to be dropped, one can simply drop all the adapters corresponding to the shards containing its samples. There are, however, two main disadvantages to this approach. On the implementation side, there is an increase in storage/inference scaling with the multitude of adapters to ensemble. On a more fundamental level, since each adapter is trained independently on a fraction of the available data, it forfeits synergistic information that may be present in data stored in different shards.  As the number of shards grows, adapters are trained on increasingly impoverished samples, leading to degraded performance relative to monolithic training (see \cref{fig:shard_synergy}).

Thus, practitioners have to choose a trade-off between increasing the accuracy of the model and minimizing the expected cost of forgetting a sample. In this work, we show that naive sharding, that is uniform partitioning of the training set, is suboptimal in many realistic use cases. Instead, shards should be constructed to maximize synergistic information available at training time. To that end, we introduce \name (Synergy Aware Forgetting Ensemble),
an algorithm that leverages synergistic information between data shards to maximize accuracy for a given forgetting cost.
We show that \name allows training on highly sharded data (with as many as 256 shards) with a $14\%$ accuracy boost over uniform sharding for the same forgetting cost.

\begin{figure*}[ht]
    \centering
    \includegraphics[width=0.85\linewidth]{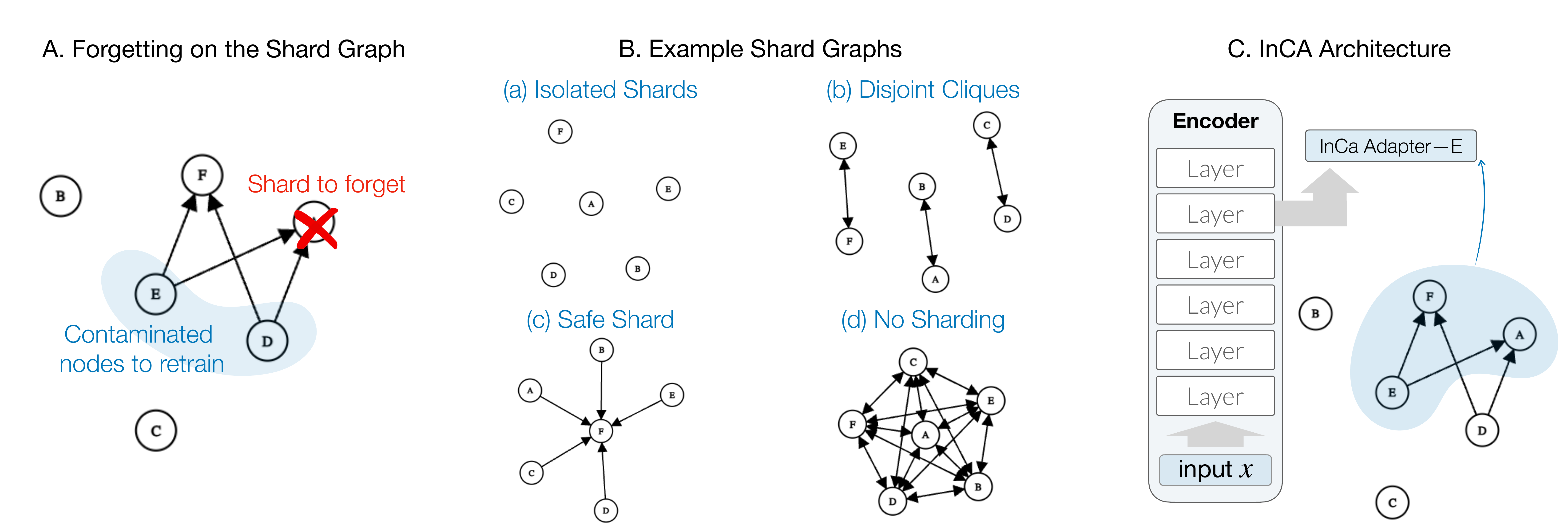}
    \caption{\textbf{(A) Forgetting on the Shard Graph.} Upon a forget request, the node containing the sample to be forgotten and all sub-models with outbound connections to the node are re-trained.
    \textbf{(B) Example Shard Graphs.} Prototypical examples of the Shard Graph include \textit{(a) Isolated Shards:} (instant forgetting of whole shards, low accuracy).  \textit{(b) Disjoint Sharding:} (fast forgetting, competitive accuracy). \textit{(c) Safe Shard:} one shard not likely to receive forget requests connected to all nodes (instant forgetting, competitive accuracy in appropriate settings).  \textit{(d) No Sharding:} (forgetting is infeasible, ideal performance). \textbf{(C) InCA Architecture.} 
    We use lightweight cross-attention adapters acting upon the representations of a pretrained transformer encoder.  Each node in the graph is associated to an InCA adapter trained on the union of the data of a node and all its outbound edges.
    }
    \label{fig:graph_examples}
\end{figure*}

To realize this, we introduce the notion of a Shard Graph (SG) (see \cref{fig:graph_examples}). Each node of the graph is a small shard of data in the original dataset. (Directed) edges between shards indicate that the corresponding adapter is not trained in isolation as usual, but also uses data from the shards it connects to, thus increasing the accessible synergistic information. On the downside, when a forget request is received, all nodes pointing to the forgotten node must be retrained, increasing the expected forgetting cost. Since different shards have different likelihood of receiving forget requests and contain different information, the problem becomes how to draw connections that maximizes the synergistic gain while minimizing the increase in expected forgetting cost. For example, shards that are unlikely to contain samples that need to be forgotten (e.g., a shard that contains synthetic data or highly vetted data) can be simultaneously  connected to many other nodes significantly increasing accuracy without increasing the expected forgetting cost.

A SG may, however, have hundreds of nodes, thus raising the problem of how to train and perform inference with hundreds of adapters. \name addresses this problem using InCA adapters \cite{inca}. These are small modules connected to a frozen pre-trained backbone, which can be trained in parallel at less than the cost of fine-tuning a single monolithic model. \name can thus handle orders of magnitude more shards than other state-of-the-art algorithms. While the use of lightweight adapters may reduce accuracy, we found the loss in accuracy negligible when the chosen backbone is well aligned with the user data, and it is offset by the large reduction in forgetting cost.

We then present two extensions of \name. First, we consider the case where information about individual samples from connected shards is bounded using ($\epsilon$, $\delta$)-Differential Privacy, which can be interpreted as using a weighted graph with edges of weight $\epsilon$. In situations where ($\epsilon$, $\delta$)-DP is an acceptable guarantee, this formulation can further reduce the expected forgetting cost by avoiding the need to always retrain all connected shards after a forget request. Second, we show that \name can be used to serve à-la-carte models \cite{Bowman_2023_CVPR}, where each user can access a model trained only on a user-specific subset data, as long as the connectivity of the SG is compatible with the user access rights.

Empirically, we show that SAFE can reduce the cost of forgetting samples by an order of magnitude compared to similar forgetting algorithms, while maintaining a similar accuracy. While evaluation of forgetting algorithms is usually performed on simple datasets such as CIFAR, we show that SAFE can be applied to more complex and nuanced fine-grained computer vision benchmarks.

\section{Related Work}
The forgetting/machine-unlearning problem \cite{7163042,ginart2019making} focuses on removing information from a trained machine learning model. One line of work \cite{sisa,arcane,legonet,kochno,kumar2022privacy, Bowman_2023_CVPR} involves splitting the dataset into multiple shards and training separate models for each shard. In such settings forgetting a sample will only affect the model trained on the subset containing the data.  In \cite{linearfiltration} the authors enable forgetting for logistic regression by allowing the model to forget a class. The recent work \cite{arcane} achieves a 1-class linear classifier trained on only a single class by making predictions based on class-centroids. The authors of \cite{legonet} propose the model LegoNet which generalizes \cite{sisa} by modifying the ensembling procedure to a select-and-average approach with the top-$k$ most relevant sub-models in an instance-dependent manner. Such methods provide instant forgetting by removal of the corresponding model with low re-training time, however, they suffer from increased inference costs directly correlating with the number of shards. 
\par
Another line of work involves forgetting for a single deep network. Forgetting is difficult for deep networks \cite{Golatkar_2020_CVPR, golatkar2020forgetting} due to the highly nonlinear nature of the parameterization and nonconvexity of the learning problem. Such works often use approximations to the non-linear landscape of a deep network to propose unlearning methods. In \cite{guo2019certified} the authors perform a newton update for unlearning, while in \cite{mixedprivacyforgettinggolatkar,achille2021lqf} the authors train a linearization of a ResNet-50 starting from a pre-trained model.  The linearity of the parameterization, enables convexity of the optimization so that one can forget specific samples and obtain upper bounds on the mutual information after a certain number of forgetting steps. However, such unlearning methods are approximate, and as a result require complete re-training of the model after a fixed number of iterations when the privacy budget is consumed. The works of \cite{wu2020deltagrad,thudi2022unrolling} provide algorithms and seek understanding of unlearning by caching models, and unrolling the iterative steps of the optimization problem.
\par
The unlearning problem has also been explored in federated setups \cite{cao2022fedrecover,liu2020federated,wang2022federated,liu2022right,gong2022forget} where a client withdraws its data and its influence needs to be removed from the global model and other clients. This approach is different from sharding based methods where there is no notion of a centralized model, and a set of weak models need to be ensembled. Some recent works \cite{neel2021descent,gupta2021adaptive,chourasia2022forget,ullah2021machine,sekhari2021remember} provide stochastic unlearning algorithms (similar to \cite{abadi2016deep}) with rigorous theoretical guarantees similar to the probabilistic guarantees in differential privacy \cite{dwork2014algorithmic} or algorithmic stability. \cite{gupta2021adaptive} showed that the ``perfect" unlearning algorithms like \cite{sisa}, provide perfect forgetting only for uniform requests, and not adaptive requests. Unlearning guarantees of approximate algorithms decay with the number of forget requests, which require full re-training of the model eventually. In one extension of SAFE, we provide a mixed learning algorithm (each shard is non-private with respect to itself, but private for other shards) using differential privacy \cite{abadi2016deep,golatkar2022mixed}. The work of \cite{chen2022graph,zhu2023heterogeneous,chien2022certified,pan2022unlearning} studies the problem of unlearning on graph modalities for tasks such as edge prediction; we note our method is concerned with typical modalities albeit the synergy between data subsets is described via a directed graph (see Sec. \ref{sec:shard_graphs} and description of the SG).

\section{Forgetting on Shard Graphs}
\label{sec:forgetting_shard_graphs}
In this section we introduce the notion of a Shard Graph and the flexible forgetting approach it enables. We then show that existing shard-based forgetting algorithms can be re-interpreted as a degenerate case where the graph does not contain any edges. In the next section we will introduce SAFE, a forgetting algorithm that can take full advantage of the provided graph.

\paragraph{Shard Graphs.}
\label{sec:shard_graphs}
For ease of notation, in this section we will assume our label space is the discrete set $\mathcal{Y} = \{1, \ldots, K\}$. The Shard Graph (SG) is a directed graph $G = (V, E)$ where the set of nodes  $V:= \{S_1, \ldots, S_n\}$ denotes different data sources or data shards $S_i \subset \mathcal{X} \times \mathcal{Y}$, which are defined by the user based on their application. We use a directed edge $(S_i, S_j) \in E$  between two shards (i.e., ``$S_i$ points to $S_j$'') to denote that when training the adapter corresponding to $S_i$ we also allow access to data from $S_j$ (see later). In all graphs we consider we assume implicitly that each node $S_i$ has a self connection, $(S_i, S_i) \in E$, that is, the adapter of a shard is always trained on that shard's data.

\paragraph{Definition of forgetting.} Consider an algorithm $\algo$ (possibly stochastic like SGD) that, given a shard graph $G$ as input, outputs a model trained on the shards of $G$. 
In this case, we denote with $\Theta \mapsto \mathbb{P}(\algo(G) \in \Theta)$ the probability distribution of possible models produced by $\algo$ given a graph $G$. Let $G'$ be a shard graph where some of the data has been removed (for example, removing an entire node/shard or removing samples of the shard). For a training algorithm $\algo$, we define a forgetting procedure $U(\algo(G), G')$ which takes a trained model $M=\algo(G)$ and the reduced graph $G'$ to output a new model $M'$ which is indistinguishable from a model trained directly on $G'$. 
To incorporate the stochasticity of $\algo$, we require that the distribution of models outputted by the forgetting procedure $U(\algo(G), G')$ matches $\algo(G')$:
\[
\mathbb{P}(U(\algo(G), G') \in \Theta) = \mathbb{P}(\algo(G') \in \Theta) \hspace{2mm} \text{for all events }\Theta.
\]
Note that a forgetting algorithm can always trivially satisfy this by ignoring $\algo(G)$ and retraining from scratch on the reduced graph $G'$, that is, setting $U(\algo(G), G') = \algo(G')$. Doing so, however, is expensive especially if $G'$ contains lots of data. The quality of a forgetting procedure is gauged by its ability to minimize the cost of forgetting while maintaining high accuracy.

\paragraph{Independent shard forgetting.} 
Given a shard graph $G$, a trivial but efficient forgetting approach \cite{sisa,Bowman_2023_CVPR,arcane} is to discard the edges and train a separate adapter on each node $S_i \in V$.
More precisely, let $A(S_i)$ denote the adapter trained on the data in $S_i$. The model corresponding to the graph $G$ is then the ensemble of all the adapters:
\begin{align}
\label{eq:independent-training}
\algo(G) = \operatorname{ensemble}(\{A(S_i)\}_{S_i \in V}).    
\end{align}
The ensembling procedure is application specific, for instance, in classification we average the logits of the models.

Such $\algo$ admit a simple forgetting procedure $U$. If requested to forget an entire node/shard in $G$, we simply need to drop the corresponding adapter incurring a constant expected forgetting cost $O(1)$. When requested to forget a sample (or subset of samples) from a shard, we only need to retrain the corresponding adapter on the remaining data.
Assuming that the cost of training an adapter is linear in the size of its training data, the expected cost to forget one sample is $O(|S|)$, where $|S|$ is the expected shard size.

When the data is divided into small shards, the expected forgetting cost is a fraction of the cost of retraining the whole model from scratch (all shards). However, as we see in Fig.~\ref{fig:shard_synergy}, the accuracy of an ensemble model trained on many small shards (SISA, orange curve) can be significantly lower than the accuracy of a single model trained on all the data simultaneously (right-most point in the curve). This can be attributed to the loss of synergistic information when training independently on many small shards.

\begin{figure}[t]
    \centering
    \includegraphics[width=0.85\linewidth]{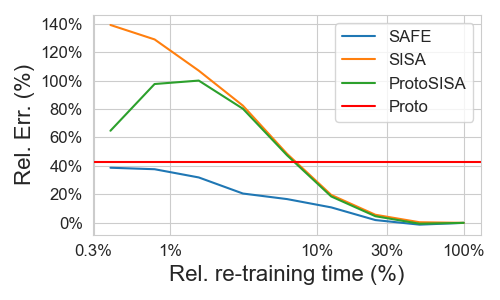}
    \caption{
    \textbf{Accuracy vs. expected cost of forgetting.} By reducing the number of shards one can improve the error (y-axis) at the expense of increasing the training time per forget request (x-axis).  We show while the error of uniform sharding methods (SISA, orange), (ProtoSISA, green) grows significantly at fine-sharding scales, SAFE is able to maintain low error and uniformly outperforms the baseline of classification based on class prototypes (blue line). We report the avg.\ for the datasets in Tab.~\ref{tab:safe_vs_other_unlearning}.
    }
    \label{fig:shard_synergy}
\end{figure}

\section{\name}
\label{sec:safe}

\name aims to reduce the impact of the loss of synergistic information by allowing shards connectivity in a way that does not significantly increase the expected forgetting cost while improving performance. Our proposed method can be formulated easily in our proposed shard graph formalism.

Let $G$ be a shard graph. Rather than training an adapter independently on the data of each shard, \name trains an adapter using also all the data contained in the connected shards. Formally, the model produced by \name is:
\begin{align}
\label{eq:safe-training}
\algo(G) = \operatorname{ensemble}\Big(\Big\{A\big({\textstyle\bigcup {N(S_i)}}\big)\Big\}_{S_i \in V}\Big)
\end{align}
where $\bigcup {N(S_i)}$ denotes the union of all the data contained in the shards connected to $S_i$  -- that is, its outbound neighborhood $N(S_i)$ (see \cref{fig:graph_examples}C). Note that, when the graph $G$ doesn't have any edges, \cref{eq:safe-training} reduces to \cref{eq:independent-training}.

Training on the union of data from connected shards exploits the synergistic information among those shards to improve accuracy at the expense of increasing the forgetting cost. However, depending on the structure of the data and the likelihood of forget requests, the increase in accuracy can greatly outweigh the increase in forgetting cost. 

\paragraph{Forgetting with \name.} Unlike before, if a sample in a shard is deleted, we need to retrain not only the adapter corresponding to that shard, but also the adapters of all shards pointing to it.  The cost of the procedure scales with the total amount of training data that needs to be revisited to train the adapters. Letting $x \in S_i$ be a sample to be forgotten, we can write as
\begin{equation}
    \label{eq:retrain-data}
    M_{x} :=  \bigcup \{ {\textstyle \bigcup} N(S_j)  : S_i \in N(S_j) \}
\end{equation}
the total data needed to retrain the adapters of all shards $S_j$ that point to $S_i$. Hence, the expected cost of forgetting depends on the expected size of $M_x$, which in turns depends on the graph topology. We now analyze some interesting cases, which will inform our experimental setup.

\textit{Random connectivity.} Suppose each node is connected to $d$ other nodes uniformly at random. Then in expectation (see Appendix) we have:
\[ \mathbb{E}|M_x| = \Theta(|S| d^2).  \]
Which scales quadratically with the degree.

\textit{Partition in disjoint cliques.} Consider now the case where the graph is partitioned into disjoint cliques of size $d$ (see \cref{fig:graph_examples}B\,(b)). In this case, all the $N(S_j)$ in \cref{eq:retrain-data} perfectly coincide and the union corresponds to $N(S_j)$. This leads to an expected forget cost that scales with the size $d$ of the shard's clique
\[
\mathbb{E} |M_x| = d \cdot |S|.
\]
In particular, compared to the random connection case, the cost is linear instead of quadratic in the degree of the nodes.

Based on this analysis, we focus on graphs that can be represented as a union of disjoint cliques, as it leads to lower cost of forgetting while allowing the same amount of connectivity. 
Finally, the case where an entire shard is deleted, rather than a single example, has similar analysis. The corresponding adapter needs to be removed from the ensemble, and all adapters of shards pointing to it need to be retrained. The expected cost scales in the same way as before as a function linear with $d$.

\paragraph{Refined Shard Graph.}
In practice, it may often happen that different shards contain samples from different classes, or even from a single class.
To deal uniformly with all the cases --- and to enable finer level of sharding and thus faster forgetting --- it is convenient to
restrict the adapters of \name to binary classifiers for each label present in the adapter's training set. This can be seen as refining the graph so that each node contains data from only one class (its positive samples) with the negative examples coming from the connected nodes.
Formally, given a shard $S \subset \mathcal{X} \times \mathcal{Y}$ we define the label occurrence map $L(S) = \{y : (x, y) \in S\}$.
Furthermore, for each label $k \in \mathcal{Y}$, we define the refined shard $S^{(k)} = \{(x, y) \in S : y = k\}$.
Then given a shard graph $G=(V, E)$ we construct a refined vertex set
\begin{gather*}
   V' := \{ S^{(k)} : S \in V, \quad k \in L(S)\}
\end{gather*}
and edge set
\begin{align*}
   &E' := \{ (S^{(h)}, S'^{(k)}) : (S, S') \in E, h \in L(S), k \in L(S')\}.
\end{align*}
Hence, for each node $S^{(i)} \in V'$ we train a binary classifier where the positive examples come from $S^{(i)}$ and the negative examples come from the all the connected nodes.

\begin{figure}
    \centering
    \includegraphics[width=0.8\linewidth]{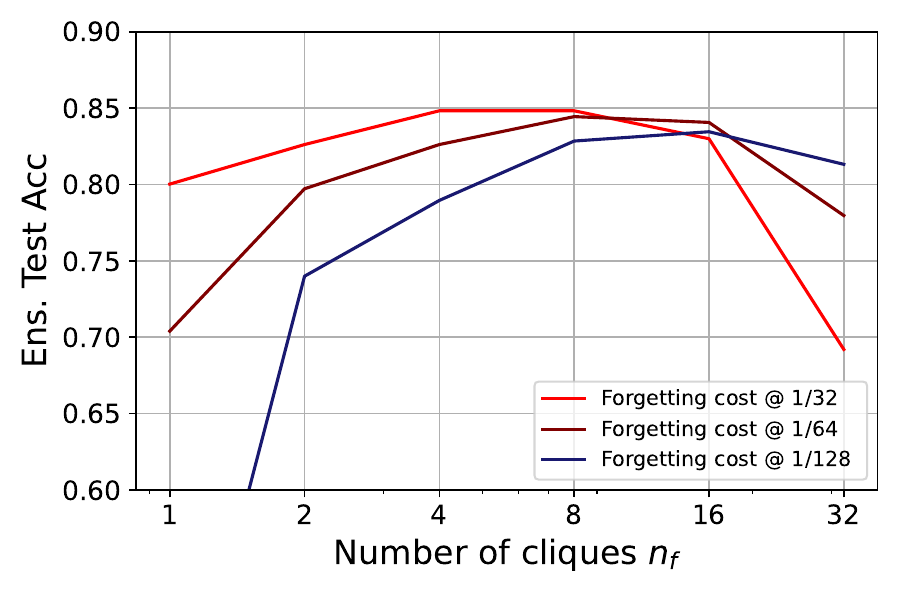}
    \caption{
    \textbf{Shard graph class composition} We report the test accuracy for MIT-67 accuracy under different shard topologies. The different curves correspond to different effective re-training times ($1/n$) as compared with standard training (smaller means faster). The points on each curve correspond to partitioning in different numbers $n_f$ of fine cliques.}
    \label{fig:pareto}
\end{figure}

\section{Synergistic Bilevel Sharding}
\label{sec:bilevel-sharding}

We now tackle the question of how to generate a graph structure that increases the synergistic information for a given expected forgetting budget. Forgetting methods usually split the data in uniform shards. However, we note that samples from the same classes, or related classes, have more synergistic information, so should preferentially appear in the same clique, even if this means that the clique may not contain all classes due to size constraints (note that, as mentioned before, we do not need the adapters to see all classes since each is a binary one-vs-all classifier). Following this intuition, we suggest the following strategy: We first split the data in $n_c$ disjoint ``coarse'' shards using class-balanced subsampling. For each coarse shard, we draw the nodes corresponding to each class and partition them randomly into cliques of size $d$, i.e., each clique contains examples from $d$ classes. This results in a number $n_f = n_\text{classes}/d$ of ``fine'' cliques per coarse shard, for a total of $n = n_c \cdot n_f$ cliques. We note that the expected forgetting costs scales with $1/n$, i.e., the amount of data in each clique. Partitioning the data uniformly, as commonly done, is equivalent to selecting $n_f=1$.

\textbf{Optimal graph structure.} Since the expected forgetting cost is the same as long as the product $n=n_c \cdot n_f$ remains constant, we check what $n_f$ produces the best model for a fixed $n$. Low $n_f$ increases the variety of the classes in each shard at the expense of loss of synergistic information, while high $n_f$ increases the synergistic information but each model sees a smaller subset of classes. In  \cref{fig:pareto} we plot the trade-off on the MIT-67 dataset. We see that indeed uniform sharding ($n_f=1$) is never optimal, and that using a higher (but not too high) values of $n_f$ is essential to get good accuracy when aiming for a low forgetting time, and can improve accuracy by more than 15\%, thus supporting the design choices for our method. Random partitioning across coarse and fine levels also makes our method robust against non-adaptive forget requests.\footnote{When the graph structure is data dependent, DP techniques may be required to prevent information leakage from it.}

\section{Efficient implementation of \name}
\label{sec:implementation}
Training independent large models on hundreds of shards leads to prohibitive storage and inference costs, and increases the expected cost of retraining for each forget request. This suggest using a shared backbone and a series of small shard-specific adapters. We use a variant of Open-InCA cross-attention adapters \cite{inca}, which lend themself easily to massive parallel training and inference (See Appendix \ref{sec:open_inca_safe} for details) while also providing high-accuracy on difficult fine-grained classification tasks. With Open-InCA adapters, we can efficiently compartmentalize the training data of each shard to its corresponding Open-InCA class parameters. 

\textbf{InCA adapters \cite{inca}.} Let $\mathbf{z} = (z_1, \ldots, z_n) = f_w(\mathbf{x})$ be the activation map produced by a frozen vision encoder $f_w(x)$. Given a set of learnable class-specific query tokens $\mathbf{q} = (q_1,\ldots, q_K)$, the output $\mathbf{y} = \ca_\theta(\mathbf{z}, \mathbf{q}, \mathbf{v})$ of InCA adapters $\ca$ are the logits $\mathbf{y} = (y_1, \ldots, y_n)$ defined by
\begin{align*}
y_i &= v_i \cdot e_i \\
\mathbf{e} &= \operatorname{cross-attention}_\theta(\mathbf{z}, \mathbf{q})    
\end{align*}
where $\theta$ are the parameters of the cross-attention layer and $\mathbf{v} = (v_1, \ldots, v_k)$ is a learnable set of vectors $v_k$ (which can be interpreted as binary linear classifiers). To keep the notation uncluttered, in the following we write $q_i$ to denote both the query vectors and the corresponding classification vector $v_i$. An important property of Open-InCA for our application is compositionality.  Let $[\mathbf{q}, \mathbf{q'}]$ denote the concatenation of $\mathbf{q}$ and $\mathbf{q'}$, then
\begin{equation}\label{eq:inca_composition}
\ca_\theta(\mathbf{z}, [\mathbf{q}, \mathbf{q}']) = [\ca_\theta(\mathbf{z}, \mathbf{q}), \ca_\theta(\mathbf{z}, \mathbf{q}')].
\end{equation}
This suggests that, rather than training a separate model on each node, we can train individual node specific $q_i$ together by concatenating them to obtain the final model. In addition to the compositionality by learning different queries $q_i$ for each classifier, Open-InCA remains expressive in selecting sub-task specific representations for each classifier.

\textbf{Applying InCA to SAFE.} We freeze and use the same cross attention weights $\ca_\theta$ which are shared across all nodes, with frozen parameters $\theta$.\footnote{We find that simply using a random initialization for $\theta$ provides a good and unbiased performance across tasks.}  Then, for each reduced shard $S_k^{(i)}$ we create corresponding queries $q^k_i$, and train the resulting binary one-versus-all classifier
\[
y_i^k(\mathbf{z}) = \ca_\theta (\mathbf{z}, q^k_i)
\]
on the data of all connected shards using a binary cross entropy (BCE) loss.  If there is a clique of nodes in the refined shard graph, we can group them together and train simultaneously instead with the cross entropy loss.  At inference time, we utilize the compositionality of the $\ca_\theta$ adapters, by computing the concatenation of all $q^k_i$ sharing the computation of all the adapters in a single forward pass and easily ensemble the resulting logits, leading to the final logits,
\[
y_i(\mathbf{z}) = \operatorname{mean}_k\big( \ca_\theta(\mathbf{z}, \mathbf{q}, \mathbf{v} )\big).
\]
In the equation above the embedding $\mathbf{z} = f_w(\mathbf{x})$ is of a test sample $\mathbf{x}$. For classification, we select the class corresponding to the highest logit.  Training each $q_i^k$ and $v_i^k$ sequentially is still expensive, and wastes computation since the same sample can be used as a negative to train multiple connected nodes. Instead, since for the BCE loss the gradients of different $q_i$ are independent, we can train all the $q_i$ at the same time over a single epoch on the whole (re)-training set, provided the loss is appropriately masked to prevent information from unconnected shards from transpiring into each $q_i$ (see Appendix \ref{sec:open_inca_safe} for details). Further since InCA does not backpropagate through the backbone $f_w$, we can also pre-compute and store the embeddings $\mathbf{z}$. Using all these techniques, we are able to train hundreds of shards at the same time in under 10 minutes on a single GPU. 

\section{Prototypical classifier}
\label{sec:proto}

Class prototypes are another viable approach to forgetting used by ARCANE \cite{arcane}. Given a dataset $D=\{(x_i, y_i)\}_{i=1}^N$ and an embedding $z = f_w(x)$, we define the prototype of the $k$-th class as
\[
p_k = \frac{1}{N_c} \sum_{(x,y) \in D^{(k)}} f_w(x)
\]
where $D^{(k)}$ are the samples of class $k$ and $N_c =|D^{(k)}|$. We can then construct a simple linear classifier
\[
y_k(z) := d_\text{cos}(z, p_k),
\]
where $d_\text{cos}$ denotes the cosine distance. Such ``prototypical'' classifiers allow instantaneous forgetting: to forget a training sample $(x_i, y_i)$ we just need to remove it from its class prototype
$p_{y_i} \mapsto \frac{ N_c\,  p_{y_i} - f_w(x_i)}{N_c-1}$. On the other hand, this classifier has suboptimal classification accuracy compared to a trained classifier on large shards of data.

However, when the shards consists only of a few samples, classifiers trained on individual shards may overfit. In such cases, the prototypical classifier can be used to provide an inductive bias \cite{snell2017prototypical}. Since the added computational and space complexity to use the prototypical classifier is negligible, we combine it into the \name model using the following expression:
\[
\operatorname{SAFE}(z) = (1 -\lambda) \cdot M(G)(z) + \lambda \cdot  \operatorname{Proto}(z),
\]
where $M(G)$ is the ensemble model in \cref{eq:safe-training}, $\operatorname{Proto}$ denotes the prototype-based classifier and $\lambda = \exp\big(-\frac{d \cdot |S|}{100}\big)$ is an interpolation weight that relies more on the prototypical classifier when the amount of data $d|S|$ used to train each adapter is small.
\begin{table*}[t]
    \centering
\includegraphics[width=0.99\linewidth]{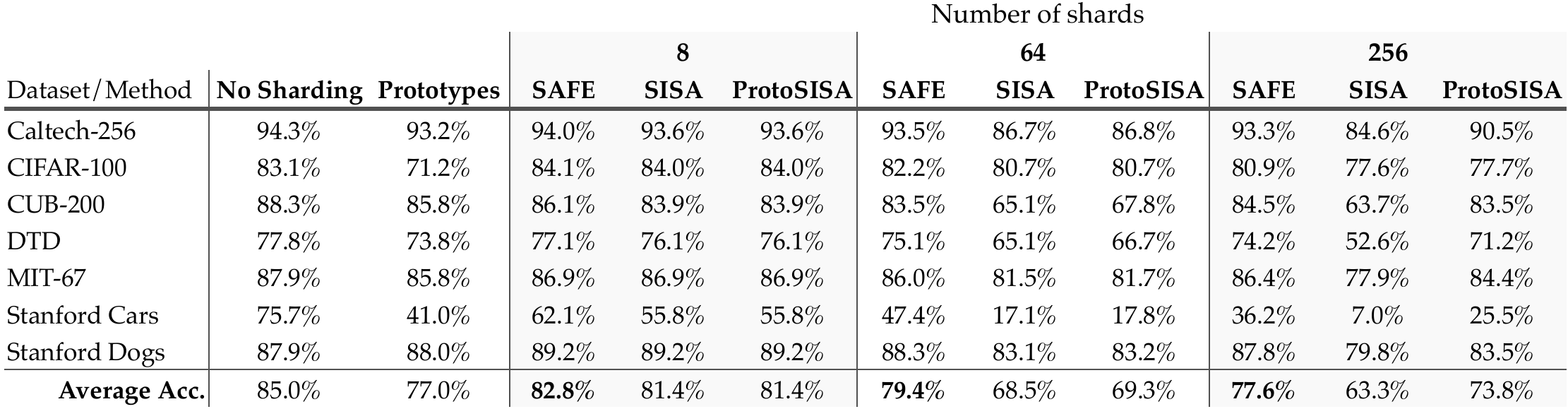}

\vspace{0.5em}

 \caption{\textbf{Accuracy of Unlearning Approaches:} We report accuracy and re-training efficiency (measured by sharding level for re-training) on a diverse set of visual classification datasets.  Forgetting a sample is equivalent to re-training the shard containing the sample.  The retraining time is inversely proportional to the number of shards.  \name allows sharding up to 256 subsets without significantly compromising accuracy. For each level of sharding (for a fixed number of shards) we try different SG topologies and report the best results in the table.  We note that for 8 shards the accuracies for SISA and ProtoSISA are the same due to $\lambda$ being exponentially small when the shards are large.}
\label{tab:safe_vs_other_unlearning}
\end{table*}

\section{Extensions of SAFE}
\paragraph{Stochastic forgetting for reduced cost.}
\label{sec:dp-forgetting} In the previous sections we presented forgetting approaches that use the edges in the shard graph to define complete usage of a shard or no usage (when there is no connection). Below we present the notion of limited shard information defined via differential privacy (DP), in which data of different nodes (data shards) of the graph are shared to a sub-model in a differentially private fashion. This can be defined by an edge weight that bounds the information shared about each sample, and can be interpreted as the probability a sample is identified in the training set. Specifically consider the binary classifier sub-models introduced in \cref{sec:implementation}, for the positive-negative samples defined by the graph topology we assign complete usage to the positive node and limited usage to the outbound connections of the negative nodes. This corresponds to a special case of mixed DP \cite{mixedprivacyforgettinggolatkar}, which results in a simple training algorithm (\name-DP) where during each epoch, each node is trained with its own data without privacy, and with data from the neighbouring node using DP. This algorithm satisfies an approximate definition of forgetting, more precisely, $\algo(G)$ is called an $(\alpha, \beta)$-unlearning algorithm if
\[\mathbb{P}(U(\algo(G), G') \in E) \leq  e^{\alpha} \mathbb{P}(\algo(G') \in E) + \beta\] for all events $E$. This definition measures the privacy leakage in terms $(\alpha, \beta)$ using group DP. We enable each user to specify their target $\alpha_{\text{target}}$ (such that $\beta < 1$), and employ privacy accounting to identify budget overflow ($\beta > 1$) for sequential forget requests. 
Such an algorithm is useful in adversarial conditions for protection against worst case adaptive forget requests.

\paragraph{\`A-la-carte models via the Shard Graph.}
The problem of constructing a unique model for a particular user that only uses data consistent with their access permissions and personal preferences, \ie the ``\`a-la-carte learning" problem \cite{Bowman_2023_CVPR}, can be solved using SAFE.  For a given user one can identify which nodes in the graph the user is unable or unwilling to access.  Based on this, one can ensemble only the adapters trained on the nodes with no outbound connections to the ineligible nodes.  Alternatively, one could simulate an artificial ``forget request'' by dropping all the data that the user is unable to access and retraining the corresponding InCA adapters.  Since InCA adapters with cached activations can be trained in seconds, one could potentially perform this training efficiently to serve custom models to different users on-demand.

\section{Experiments}
\label{sec:exp}
\textbf{Model.} In our experiments  we use as encoder a VIT-L/16 transformer architecture \cite{dosovitskiy2021an}  pretrained on ImageNet-21k at 224 image resolution.\footnote{We use the  \texttt{vit\_large\_patch16\_224\_in21k} pre-trained model in the \texttt{timm} library \cite{rw2019timm}.} 
We use the InCA adapters from \cite{inca}, in particular we train the head and queries, but keep the cross-attention layer frozen and shared between all shards. We train each adapter with AdamW for 30 epochs using cosine annealing starting from $\text{lr} = 0.05$ and weight decay of $10^{-4}$. See the Appendix for full details.

\paragraph{Datasets.} 
We evaluate on the following vision classification datasets, covering fine-grained classification tasks and domain shifts with respect to the ImageNet pretraining:
CUB-200 \cite{WahCUB_200_2011}, MIT-67 \cite{mit67recognizing}, Caltech-256 \cite{griffin_holub_perona_2022}, CIFAR-100 \cite{Krizhevsky09learningmultiple}, Describable Textures (DTD) \cite{cimpoi14describing}, Stanford Cars \cite{KrauseStarkDengFei-Fei_3DRR2013}, and Stanford Dogs \cite{KhoslaYaoJayadevaprakashFeiFei_FGVC2011}. 

\textbf{Baselines.} We compare our method against two main classes of methods for forgetting: Prototypes (motivated by ARCANE\footnote{In \cite{arcane} they work with smaller models (e.g. ResNet-18 with 11M params.) and train a separate embedding for each class, whereas we take the embeddings of a fixed large pretrained transformer (305M params.).} \cite{arcane}) uses a prototype based classifier (\cref{sec:proto}), which allows constant time forgetting but which cannot be trained to improve accuracy, SISA \cite{sisa} performs forgetting by training an ensemble of models on shards created by sampling uniformly at random. This corresponds to using $n_f=1$ when structuring the bilevel sharding (\cref{sec:bilevel-sharding}). Like in \name, increasing the number $n$ of shards leads to lower forgetting time but also lower accuracy. We re-implement SISA using the same architecture and training scheme as SAFE for a direct comparison.  \cite{sisa} uses further slicing of each shard (through intermittent model checkpoints) to further speed-up re-training complementarily.  Slicing can be similarly incorporated in SAFE, but we isolate slicing from our analysis as it is an orthogonal and complementary approach that can be applied in each method (with increased storage costs).

\paragraph{Comparison of \name with baselines.}
In \cref{fig:shard_synergy} we plot the average trade-off between accuracy and forgetting time for \name and other methods from the literature (SISA \cite{sisa} and Prototypes \cite{arcane}) across the datasets in \cref{tab:safe_vs_other_unlearning} for various target forgetting times. We see that \name performs uniformly better than SISA and Prototypes for all forgetting budgets. Thanks to the ability to train, it outperforms Prototypes for higher forgetting time budgets, while it significantly outperform SISA in the low-forgetting-time regime (high-sharding) due to better use of synergistic information.

\begin{figure}
    \centering
    \includegraphics[width=0.8\linewidth]{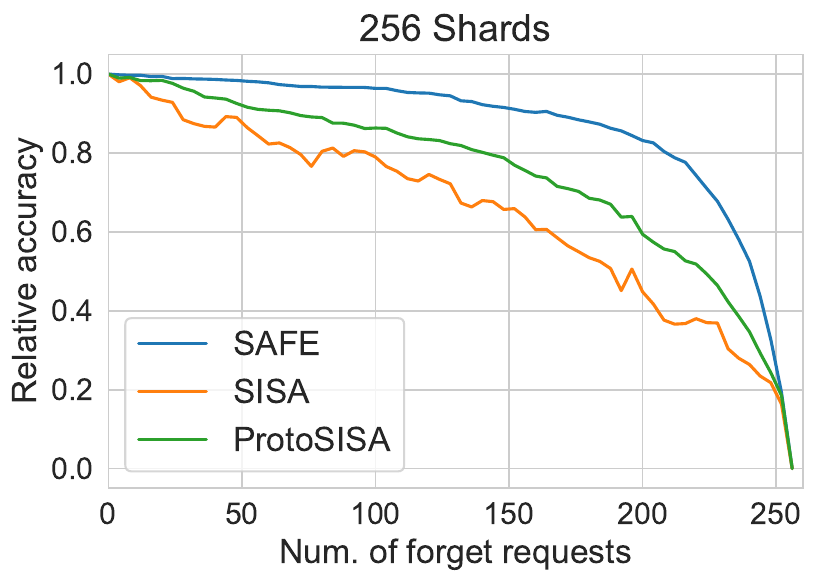}
    \caption{\textbf{Instant Forgetting.} We simulate a series of forget requests where for each forget request we drop an entire shard of data without retraining.  We plot the relative accuracy averaged across the 7 datasets for the methods SAFE, SISA, and ProtoSISA for 256 shards.}
    \label{fig:instant_forgetting}
\end{figure}

\paragraph{Domain shift.} In \cref{tab:safe_vs_other_unlearning} we see that the accuracy of \name and other forgetting methods decreases more rapidly with the sharding level for datasets that have a significantly different distribution from the ImageNet-21k pretraining of the backbone (e.g., DTD and Stanford Cars). We hypothesize that this is because the synergistic information of different samples is already contained in the backbone when the data is similar to the pre-training, hence the loss due to sharding is less influential. However, thanks to its better handling of synergistic information, we see that \name performs significantly better than the other baselines on difficult domains.

\paragraph{Ablations.} We can ask whether the better performance of \name is due to the use of prototypes, the synergistic information or both. We have seen in \Cref{fig:pareto} that synergistic information alone improves over SISA ($n_f=1$). In \cref{fig:shard_synergy} we also show the result of a further baseline, ProtoSISA, obtained by adding a prototype classifier to SISA using the same weighting scheme used for \name. We see that while adding prototypes to SISA boosts performance, \name still significantly outperforms both other methods, showing that indeed both aspects are important.

\paragraph{Instant forgetting.}\label{par:instant_forgetting} 
In certain situations, a service provider will need to satisfy a forget request instantly and thus will need to drop an entire data source without retraining, meaning that all adapters using those samples will need to be dropped without being retrained and replaced.  To investigate how robust SAFE, SISA, and ProtoSISA are to such requests, in \cref{fig:instant_forgetting} we plot the relative accuracy of SAFE, SISA, and ProtoSISA after a series of forget requests for ensembles with 256 shards.  We see that \name exhibits a smaller decline in relative accuracy in the presence of forget requests, and uniformly outperforms SISA and ProtoSISA.  For all methods the marginal decline in accuracy increases as the number of forget requests increases, suggesting that the ensemble becomes more sensitive as additional predictors are dropped.

\begin{figure}
    \centering
    \includegraphics[width=0.95\linewidth]{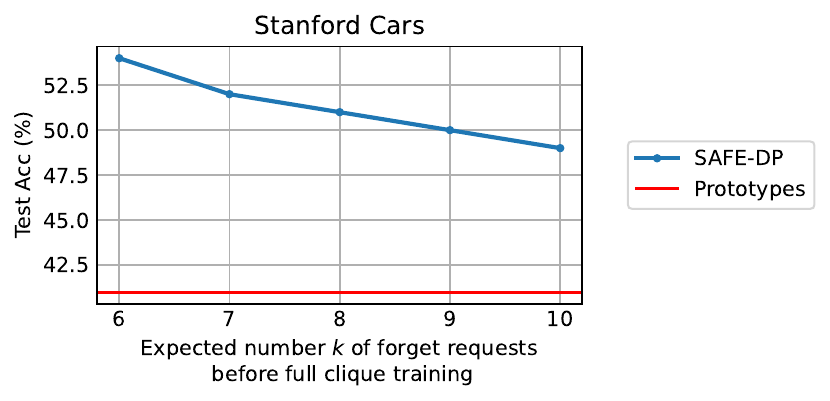}
    \caption{\textbf{Stochastic forgetting.} Training with DP on the outbound nodes we can satisfy $k$ forget requests before having to retrain the whole graph, at the expense of lower accuracy.}
    \label{fig:dp-forgetting}
\end{figure}

\paragraph{Stochastic forgetting.} In \cref{sec:dp-forgetting} we propose a DP-based mechanism that allows shards to receive up to a given number $k$ of forget requests without having to retrain a whole clique, at the expense of decreased accuracy due to the bound on information imposed by DP. In \cref{fig:dp-forgetting} we show this trade-off on a dataset with significant distribution shift (Stanford Cars). The model trained, for example, with SAFE-DP ($k=8$) provides lower accuracy than SAFE ($n=8$) ($51.0\%$ instead of $62.1\%$), but has a better worst case guarantee (it can accommodate 8 sequential forget requests before retraining). Thus \name-DP may trade-off a worst case privacy guarantee for model accuracy. This is especially useful in settings with adversarial forget requests \cite{gupta2021adaptive}.

\section{Multi-Domain synergy experiments}
Next we study \name's ability to harness synergistic effects for tasks involving multiple domains. For this we consider the 4-domain challenge, DomainNet-126 \cite{Saito_2019_ICCV} that is a curated subset of the extended DomainNet suite \cite{peng2019moment}. To best evaluate synergy between domains, we ensure each domain's dataset is of the same size, and we sub-sample the training and test set of each domain to be of fixed size and with a balanced distribution of classes. This results in 6300 training samples and 2520 test samples for each domain. \par
To test the benefits of \name for the multi-domain task of predicting among the 126 different common categories, now with images coming from different domains, we consider different SG topologies.
We define the topologies tested as follows.
\begin{itemize}
    \item \textbf{SISA In-Domain} Each domain is trained separately with SISA-style uniform sharding with $16$ shards in each domain ($64$ total shards).
    \item \textbf{SISA Cross-Domain} The training datasets are merged into a single cross-domain training dataset that is used for SISA-style uniform sharding (shards are maintained to be of same size as in ``SISA In-Domain'', \eg $64$ total shards).
    \item \textbf{\name In-Domain} \name is applied to each domain separately, which creates cliques containing subsets of classes with each clique having the same re-training costs as the previous two approaches.
    \item \textbf{\name Cross-Domain} The cliques are constructed to contain synergistic connections between different domains with the overall clique size remaining the same.
\end{itemize}
In Table \ref{tab:domain_net} we evaluate the different approaches for a $64$ shard level. For in-domain topologies this corresponds to 16 shards per-domain. For \name we use $4$ coarse shards and $4$ fine shards.  We observe that by modifying the Shard Graph topology for \name and allowing for shard connections between different domains, \name is capable of learning more accurate representations at the same training and unlearning costs. 
\begin{table}
    \centering
\resizebox{0.49\textwidth}{!}{
\begin{tabular}{c|c| ccccc}
\toprule
Method &  All &  Clipart &  Sketch &  Real &  Painting \\
\midrule
SISA (in-domain) &               67.88 &              67.70 &             59.25 &           85.12 &               59.44 \\
\name (in-domain) &               73.41 &              73.77 &             66.19 &           88.89 &               64.80 \\
\hline
SISA (synergy) &               69.97 &              69.84 &             62.02 &           86.67 &               61.35 \\
\name (synergy) &               \textbf{76.11} &              \textbf{77.14} &             \textbf{68.10} &           \textbf{90.40} &               \textbf{68.81} \\
\bottomrule
\end{tabular}
}
\caption{\textbf{Multi-domain unlearning on DomainNet-126} We report the test accuracy when using different shard graph topologies on DomainNet-126.  The column ``All" corresponds to evaluation on the union of the 4 domains, whereas the other four columns correspond to the test accuracy on the single domain.}
\label{tab:domain_net}
\end{table}

\section{Conclusion}
We introduced the Shard Graph, a directed graph describing the access relations between data sources for training. This graph informs the creation of \name: an ensemble of lightweight adapters trained on the datasets specified by the graph structure. By constructing the graph to maximize synergies between datasets while minimizing connections and retraining time, we are able to handle an order of magnitude more shards than competing methods, and achieve a $14\%$ accuracy boost over competing methods for the same forgetting cost.
We conclude that maximizing synergistic information while minimizing dataset overlap is the fundamental trade-off at the core of compartmentalization-based forgetting, which so far has been under-explored.
 In some cases, the accuracy of our method may be limited by the use of light-weight adapters. However, InCA adapters demonstrate high accuracy while allowing us to efficiently train many adapters in parallel without information leakage and permit fast retraining, which make them conducive to compartmentalization-based forgetting. Rather, we find that the bottleneck in accuracy is mainly due to the loss of synergistic information due to sharding which we alleviate with \name.
Finally, optimizing the shard graph structure to utilize additional properties in the data without leaking sensitive information is an important problem which we leave to future work.

{\small
\bibliographystyle{ieee_fullname}
\bibliography{main.bib}
}

\clearpage

\renewcommand{\thesection}{\Alph{section}}
\renewcommand{\thesubsection}{\Alph{section}.\arabic{subsection}}
\setcounter{section}{0}

\newcommand{\EE}{\mathbb{E}}
\newcommand{\PP}{\mathbb{P}}
\newcommand{\parens}[1]{\left(#1\right)}
\newcommand{\brackets}[1]{\left[#1\right]}
\newcommand\abs[1]{\left\lvert#1\right\rvert}

\section*{Supplementary Material}
\section{Parallel and compartmentalized training of SAFE}\label{sec:open_inca_safe}
In the following section we review the Open-InCA architecture and its use for efficient training of compartmentalized adapters in SAFE. In our training we are able to train all adapters in parallel while having exact control of each parameter's access to samples' gradients. This approach enables us to define the Shard Graph flexibly and provide custom ``data access'' to each adapter while training all of SAFE's adapters in parallel.

\paragraph{Open-InCA.}
In the work of \cite{inca} it is shown that lightweight cross-attention adapters attached to intermediate network representations are capable of extracting relevant representations and achieving competitive accuracy relative to expensive full fine-tuning. For regular learning, given an activation map $\mathbf{z}$ of a pre-trained model, the InCA adapter structure is defined as
\begin{align*}
    \mathbf{e} &= \operatorname{cross-attention}_{\theta_1}(\mathbf{z}, [q]) \\
    \mathbf{y} &=  \operatorname{Linear}_{\theta_2}(\mathbf{e}).
\end{align*}
By learning the cross-attention layer parameters and the query $q$, the adapter is capable of extracting task-relevant representations from the activation map $\mathbf{z}$.

For more flexible learning scenarios the authors present Open-InCA, where each class prediction is computed via a separately learned query and class-head parameters,
\begin{align*}
    \mathbf{e} &= \operatorname{cross-attention}_{\theta_1}(\mathbf{z}, [q_1, \dots q_{C}]) \\
    \mathbf{y} &=  \operatorname{Diag-Linear}_{\theta_2}(\mathbf{e}).
\end{align*}
Here $\operatorname{Diag-Linear}_{\theta_2}$ is a mapping that takes the query embeddings $[e_1, \ldots, e_C]$ and acts on them ``diagonally" through the classification vectors $[v_1, \ldots, v_C]$:
\[ y_i := v_i \cdot e_i. \]
We follow the Open-InCA approach which is originally defined for class incremental learning, but for our settings we would like to achieve complete isolation of the learning of each adapter.

\paragraph{Adapter isolation.}
To achieve isolated learning of each adapter we apply Open-InCA's ``query only learning'' \cite{inca}. This amounts to not optimizing the cross-attention layer weights $\theta_1$ and only training $[q_1, \dots q_c]$ and the separate classification head parameters $[v_1, \dots v_c]$. Nonetheless, applied with the traditional $\operatorname{softmax}$ and Cross-Entropy loss, query only training \emph{still} leaks information of each sample in the batch to each adapter. Recall that $\softmax$ is defined as \[
[\operatorname{softmax}(\mathbf{y})]_i = \frac{\exp (y_i)}{\sum_{k=1}^{C} \exp(y_k)}.
\]
The $\operatorname{softmax}$'s gradient of each class prediction will be affected by the predictions of all other adapters (due to the denominator), thereby leaking information. To rectify this, we replace the $\operatorname{softmax}$ with a sigmoidal mapping $\sigma(y_i)=\frac{1}{\exp(-y_i) + 1}$, so the loss of each adapter prediction relies solely on the adapter's logit value $y_i$. Using the sigmoidal mapping enables us to use \name in more flexible learning settings. This is because now each node in the shard graph can be responsible for learning just a single binary classifier for a particular class. 

Transitioning to Open-InCA with sigmoidal loss results in hundreds or even thousands of individually learned binary classifiers. Trained sequentially the vast number of adapters will result in prohibitively long training time. 
\begin{figure}[b]
    \centering
\includegraphics[width=8cm]{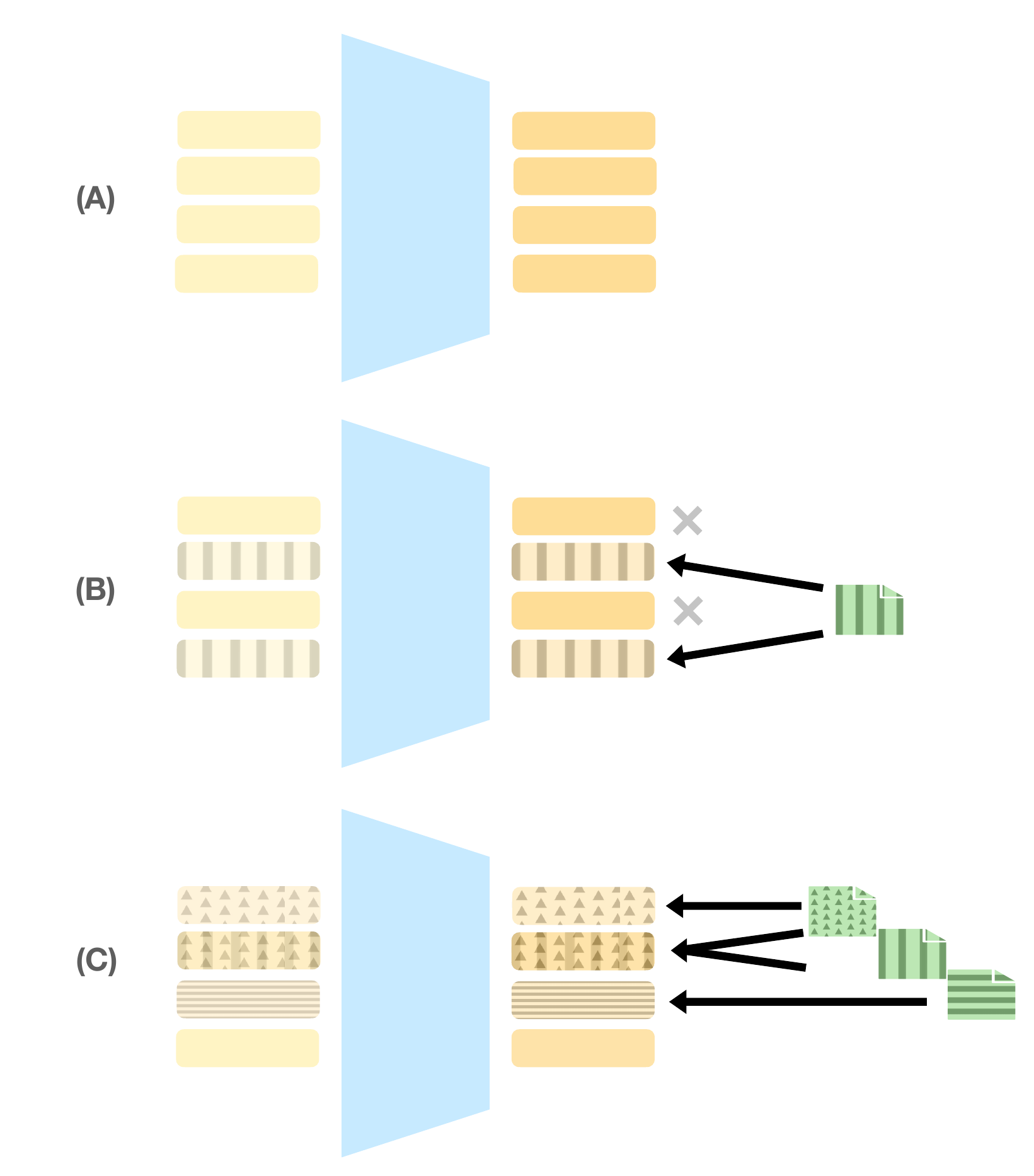}
    \caption{\textbf{SAFE Compartmentalized training with Open-InCA} We illustrate the use of Open-InCA for efficient and parallelized training with SAFE. \textbf{(A)} In Open-InCA, each learned binary classifier corresponds to its own dedicated query (yellow) and class-head (orange), in blue is the shared frozen cross-attention module. \textbf{(B)} Given a data-sample (green), we may control which adapters can be updated with its gradient, and based on the architectural choice there will be no information leakage. \textbf{(C)} Using a batch of data-samples and a graph topology of data access rights, we can learn all adapters in parallel using loss masking which automatically blocks and routes gradients to their appropriate adapters.}
    \label{fig:my_label}
\end{figure}

\paragraph{Loss masking.} 
Instead we show that we can learn all of the adapters within an Open-InCA adapter in parallel while still isolating information via loss masking. As stated in Eq. \eqref{eq:inca_composition} of the manuscript the different queries of the Open-InCA adapter can be de-composed as
\[
\ca_\theta(\mathbf{z}, [\mathbf{q}, \mathbf{q}']) = [\ca_\theta(\mathbf{z}, [\mathbf{q}]), \ca_\theta(\mathbf{z}, [\mathbf{q}'])].
\]
By using the sigmoid mapping we compute the loss for each Open-InCA binary-class adapter as 
\begin{align*}
L_{i,j} =& -\hat{y}_j^{(i)} \cdot \log (\sigma(\ca_\theta(\mathbf{z}, [\mathbf{q}_i])) \\
& - (1-\hat{y}_j^{(i)}) \cdot \log (1-\sigma(\ca_\theta(\mathbf{z}, [\mathbf{q}_i])),
\end{align*}
where $\hat{y}_j^{(i)}$ is 1 if example $j$ is a positive example for the class corresponding to adapter $i$, and $0$ otherwise.
For a batch of $B$ samples and a set of $M$ adapters, we obtain the loss matrix $L \in \mathbb{R}^{M \times B}$. Each loss entry $L_{i,j}$ is computed with frozen cross-attention parameters and the gradients are based on the learnable parameters of only the $i$th adapter (due to $\sigma$):
\[
\frac{\partial L_{i,j}}{\partial q_k} = 0 ~ \text{for } k \ne i.
\]
This implies any summation with $\{a_i\} \in \mathbb{R}^{M}$
\[
\ell = \sum_{k=1}^{M} a_k L_{k,j}
\]
has the property of reducing to a single term after differentiation 
\[
\frac{\partial L}{\partial q_i}  = a_i \cdot \frac{\partial L_{i,j}}{\partial q_i}.
\]
This is highly convenient in automatic differentiation packages, as it reduces our entire optimization (optimizing each adapter) into a single ``backwards'' call with the summed loss $\ell$. 
In particular, for a given SG topology we can define $A$ for a data batch with $A \in \{0,1\}^{M \times B}$ such that 
\[
A_{ij} =
\begin{cases}
    1 & \text{if shard } i \text{ has access to data sample } j \\
    0  & \text{otherwise}
\end{cases}
\]
With this we can optimize the entire SAFE ensemble in parallel in a  compartmentalized manner by directly optimizing the objective $\ell = \sum(L \odot A)$ where $\odot$ denotes the Hadamard product and $\sum$ is computed over all the matrix entries. 

\paragraph{DP-SAFE.}
For the mixed-privacy DP-SAFE method we can follow a similar process in which we compute a private loss and a non-private loss using 2 separate masking rules defined via $A_{\text{DP}} \in \{0,1\}^{M\times B}$ and $A_{\text{direct}} \in \{0,1\}^{M\times B}$ accompanied by respective optimizers DP-SGD and AdamW.

\begin{figure*}
    \centering
    \includegraphics[width=\linewidth]{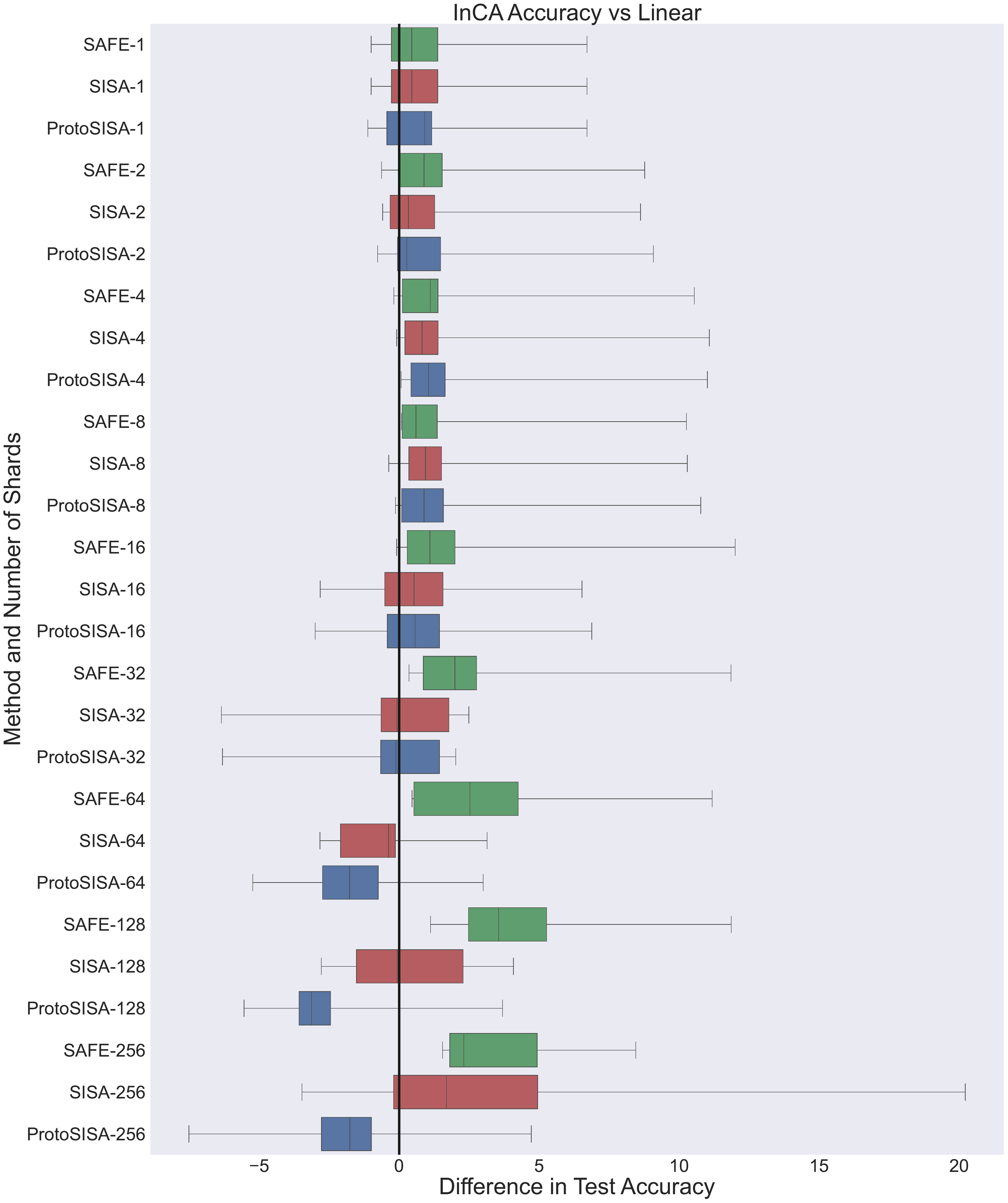}
    \caption{\textbf{InCA vs. Linear.}  We report the difference in test accuracy between using the InCA model and using a linear model for the methods SAFE, SISA, and ProtoSISA at sharding scales $1, 2, 4, 8, 16, 32, 64, 128, 256$.  Positive values correspond to InCA having higher accuracy, negative values correspond to the linear model having higher accuracy.  Each box corresponds to 7 values, specifically the test accuracy over the 7 datasets we consider.  The boxes corresponding to \name, SISA, and ProtoSISA are colored green, red, and blue respectively.}
    \label{fig:inca_vs_linear}
\end{figure*}

\section{Additional Ablations}
\paragraph{Additional Sharding Scales.}
In \cref{tab:safe_shard_nums}, \cref{tab:sisa_shard_nums}, and \cref{tab:protosisa_shard_nums} we report the accuracies of \name, SISA, and ProtoSISA across a larger range of sharding scales, namely $2, 4, 8, 16, 32, 64, 128, 256$.  We see that \name outperforms SISA on average across all sharding scales, with the gap peaking at $14.3\%$ at $256$ shards.  Furthermore while the ProtoSISA method improves over SISA at the large sharding scales $\geq 64$, it still underperforms \name by a margin of $(3.8\%-10.1\%)$.

\paragraph{\name without prototypes.}
\name benefits from both synergy-aware sharding and the inductive bias of the prototype model.  To separate out these two factors, we consider the \name method without prototypes, which we call ``NoprotoSAFE".  In \cref{fig:noproto_safe_comps} we compare the accuracy of \name, Noproto\name, and SISA.  We see that at the large sharding scales, the use of prototypes in \name gives a modest boost in test accuracy $0.5-2.5\%$ over Noproto\name.  We see that Noproto\name still significantly outperforms SISA, with margins as high as $10-30\%$, particularly at the large sharding scales $\geq 64$.  We conclude that the synergy-aware sharding of the SAFE method is the primary factor for the boost in performance over SISA. 

\begin{figure*}[ht]
    \centering
    \includegraphics[width=0.9\linewidth]{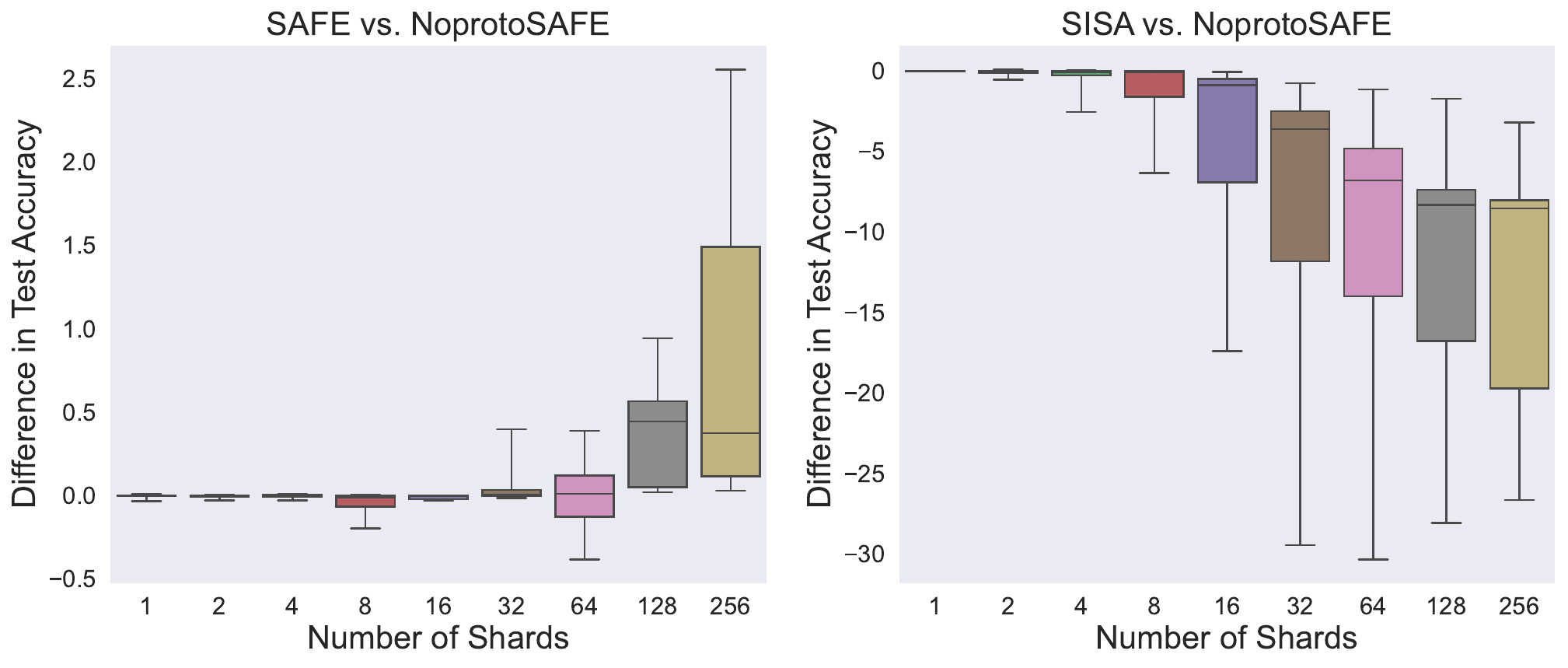}
    \caption{\textbf{\name without prototypes.} We provide comparisons against \name with and without applying prototypes.  SAFE without applying prototypes is dubbed ``NoprotoSAFE".  Each box in the subplots corresponds to 7 values, one for each of the datasets.  \textbf{(Left) \name vs. Noproto\name.} We report the difference in test accuracy between \name and Noproto\name.  Positive values correspond to \name outperforming, negative values correspond to Noproto\name outperforming.  \textbf{(Right) SISA vs. Noproto\name.}  We report the difference in test accuracy between SISA and Noproto\name.  Positive values correspond to SISA outperforming, negative values correspond to Noproto\name outperforming.}
    \label{fig:noproto_safe_comps}
\end{figure*}

\paragraph{\name and SISA using a linear model.}
Another efficient approach to adapting large pretrained models is to perform head-only finetuning, where only a linear classifier head is trained.  Thus a natural question is how do \name, SISA, and ProtoSISA perform when replacing the InCA adapters with linear models.  In \cref{fig:inca_vs_linear} we report the gain in test accuracy when using the InCA model relative to a linear model at different sharding scales for the various methods.  We see that for the \name method the InCA model uniformly outperforms the linear model at all sharding scales.  For SISA and ProtoSISA, at the larger sharding scales $\geq 64$ the linear model starts to outperform in some cases.  We suspect that SISA and ProtoSISA have difficulty training the cross-attention module when there are very few examples per shard, whereas the synergistic sharding of \name enables it to successfully train a more complex model.

\section{Analysis of the Shard Graph}
In Section \ref{sec:safe}, we theoretically analyze the expected forgetting costs of \name for different shard graph topologies. To understand the effect of different topologies, we consider a graph with a determined connectivity level, defined by assigning each node in $V$ $d$ outbound edges. When presenting \name, we concluded that when using a disjoint clique structure for the Shard Graph the expected cost to forget a sample $x$ (i.e. the number of samples required for re-training) is 
$\mathbb{E}|M_x| = d \cdot |S|$.  We note this is the optimal case and if the degree $d$ connectivity is applied uniformly at random, $\mathbb{E}|M_x|$ is in fact an order of magnitude larger $\mathbb{E}|M_x| \sim d^2 \cdot |S|$.  Below we present a formal statement and proof that $\mathbb{E}|M_x| \sim d^2 \cdot |S|$ for a graph with random connectivity.

\begin{theorem}
Suppose we have a set of nodes $V = \{S_1, \ldots, S_n\}$ where each $S_i$ is a source of data and $S_i \cap S_j = \emptyset$ for $i \neq j$.  Furthermore assume that all the sources are of the same size, i.e. $|S_1| = |S_2| = \cdots = |S_n|$.  Suppose each node is assigned $d$ outbound edges independently and uniformly (in addition to the default self-connection).  Furthermore assume that $d^2 \leq C |V|^{1/2}$ for some $C > 0$.  Then the expected number of samples $|M_x|$ needed for retraining upon a forget request for a sample $x$ is
\[ \mathbb{E}|M_x| = \Theta(|S_1| d^2).  \]
\end{theorem}
\begin{proof}
Assume we receive a forget request for a sample $x \in S_i$.  By uniformity without loss of generality we may assume that $i = 1$.  We will let $Z$ be a random variable denoting the number of samples we must retrain on whenever receiving a forget request for a sample $x \in S_1$ where the randomness is taken over the selection of the edges $E$ of the graph.  For a node $n$ let $N_{in}(n) = \{v : (v, n) \in E\}$ denote the inbound neighborhood of $n$ and let $N_{out}(n) = \{ v : (n, v) \in E\}$ denote the outbound neighborhood of $n$.  Then we have that
\[ Z = \abs{\bigcup_{n \in N_{in}(S_1)} \bigcup_{v \in N_{out}(n)} v \setminus \{x\}}. \]
We note then that
\begin{align*}
Z &\leq \sum_{n \in N_{in}(S_1)} \sum_{v \in N_{out}(n)} |S_1| \\
&\leq |N_{in}(S_1)| (d + 1) |S_1|.
\end{align*}
Therefore
\[ \EE[Z] \leq |S_1| (d + 1) \EE[|N_{in}(S_1)|]\]
We note that for each node $v \in V \setminus S_1$ that $\mathbb{I}[v \in N_{in}(S_1)]$ is a Bernoulli random variable taking the value $1$ with probability $\frac{d}{|V| - 1}$.  Furthermore by default we know that $S_1 \in N_{in}(S_1)$.  Therefore
\[ \abs{N_{in}(S_1)} - 1 = \sum_{v \in V \setminus S_1} \mathbb{I}[v \in N_{in}(S_1)] \]
is a sum of $|V| - 1$ independent Bernoulli random variables.  Thus $|N_{in}(S_1)| - 1$ obeys the Binomial distribution $Bin(|V| - 1, \frac{d}{|V| - 1})$ which has mean $d$.  Thus
\[ \EE[|N_{in}(S_1)|] = d + 1 \]
and we conclude that
\[\EE[Z] \leq |S_1| (d + 1)^2 = O(|S_1| d^2). \]
\par
Now we will prove the lower bound.  We note for $d = 1$ the statement is trivial so we might as well assume $d \geq 2$.  Since we are seeking a lower bound we may assume that after the forget request is received the entire source $S_1$ is dropped, as this only decreases the number of samples needed to retrain.  In this case
\[ Z = \abs{\bigcup_{n \in N_{in}(S_1) \setminus S_1} \bigcup_{v \in N_{out}(n) \setminus S_1} v}. \]
We will focus on estimating the conditional expectation
\[ \EE[Z \vert \hspace{2mm} |N_{in}(S_1) \setminus S_1| = k]. \]
We note that by uniformity that the distribution of $Z$ depends only on the size of $N_{in}(S_1) \setminus S_1$ and not the specific collection of nodes in $N_{in}(S_1) \setminus S_1$.  Thus we will fix a choice $\{n_1, \ldots, n_k\}$ of $k$ nodes for $N_{in}(S_1) \setminus S_1$.  Let $E$ denote the event that $N_{in}(S_1) \setminus S_1 = \{n_1, \ldots, n_k\}$ and let $\PP_E(\bullet) := \PP(\bullet | E)$ denote the probability of an event conditioned on $E$.  We note then that
\[ \EE[Z \vert \hspace{2mm} |N_{in}(S_1) \setminus S_1| = k] = \EE[Z \vert E]. \]
  Let $A_i = N_{out}(n_i) \setminus S_1$ for $i = 1, \ldots, k$.  We note that if $A_1, A_2, \ldots, A_k$ are disjoint then $Z = kd \cdot |S_1|$.  Thus we have that
\begin{align*}
&\EE\big[Z \,\big\vert \ \, E \big] \\
&\hspace{2em}\geq kd \cdot |S_1| \cdot \PP_E(A_i \cap A_j = \emptyset \hspace{2mm} \forall 1 \leq i \neq j \leq k).  
\end{align*}
Thus we proceed to lower bound
\[ \PP_E(A_i \cap A_j = \emptyset \hspace{2mm} \forall 1 \leq i \neq j \leq k). \]
The number of disjoint choices of $A_1, A_2, \ldots, A_k$ is
\[ \prod_{j = 0}^{k - 1} \binom{|V| - (k + 1) - j (d - 1)}{d - 1}. \]
The number of total choices of $A_1, A_2, \ldots, A_k$ is
\[ \binom{|V| - 2}{d - 1}^k. \]
Thus we have
\begin{gather*}
\PP_E(A_i \cap A_j = \emptyset \hspace{2mm} \forall 1 \leq i \neq j \leq k) \\
\geq \brackets{\frac{\binom{|V| - (k + 1) - (k - 1)(d - 1)}{d - 1}}{\binom{|V| - 2}{d - 1}}}^k.
\end{gather*}
Now assume that $k \in [(1 - t)d, (1 + t)d]$ for some fixed $t \in (0, 1)$.  Note then that
\begin{align*}
\brackets{\frac{\binom{|V| - (k + 1) - (k - 1)(d - 1)}{d - 1}}{\binom{|V| - 2}{d - 1}}}^k 
&\geq \brackets{\frac{(|V| - kd)^{d - 1}}{(|V| - 2)^{d - 1}}}^{k} \\
&\geq \brackets{\frac{(|V| - kd)^{d - 1}}{|V|^{d - 1}}}^{k} \\ 
&\geq \brackets{\frac{|V| - kd}{|V|}}^{kd} \\
&= \brackets{1 - \frac{kd}{|V|}}^{kd}.
\end{align*}
Now using the fact that $d^2 \leq C|V|^{1/2}$ and $k \leq (1 + t) d \leq 2d$, we have that
\begin{align*}
\brackets{1 - \frac{kd}{|V|}}^{kd} 
&\geq \brackets{1 - \frac{2 d^2}{|V|}}^{2 d^2} \\
&\geq \brackets{1 - \frac{2 C}{|V|^{1/2}}}^{2 C |V|^{1/2}} \\
&= \brackets{1 - \frac{4 C^2}{2 C |V|^{1/2}}}^{2 C |V|^{1/2}}
\\ &= e^{-4 C^2} + o(1) = \Omega(1).
\end{align*}
Thus it follows for $k \in [(1 - t)d, (1 + t)d]$ we have that
\begin{gather*}
\EE[Z \vert \hspace{2mm} |N_{in}(S_1) \setminus S_1| = k] \\
\geq |S_1| d \cdot k \cdot \PP_E(A_i \cap A_j = \emptyset \hspace{2mm} \forall 1 \leq i \neq j \leq k) \\
= \Omega((1 - t) |S_1| d^2).    
\end{gather*}
We note it then suffices to show that for some fixed $t \in (0, 1)$
\[ \PP\parens{|N_{in}(S_1) \setminus S_1| \in [(1 - t)d, (1 + t)d]} = \Omega(1). \]
We recall that for each node $v \in V \setminus S_1$ that $\mathbb{I}[v \in N_{in}(S_1)]$ is a Bernoulli random variable taking the value $1$ with probability $\frac{d}{|V| - 1}$.  Thus 
\[ \abs{N_{in}(S_1) \setminus S_1} = \sum_{v \in V \setminus S_1} \mathbb{I}[v \in N_{in}(S_1)]  \]
is a sum of $|V| - 1$ independent Bernoulli random variables.  Thus $|N_{in}(S_1) \setminus S_1|$ obeys the Binomial distribution $Bin(|V| - 1, \frac{d}{|V| - 1})$ which has mean $d$ and variance $d (1 - \frac{d}{|V| - 1})$.  Now let $W = \abs{N_{in}(S_1) \setminus S_1}$.
Then by Chebyshev's inequality for $t > 0$
\[ \PP\parens{|W - d| \geq t d} \leq \frac{(1 - \frac{d}{|V| - 1})}{d t^2} \leq \frac{1}{d t^2}.\]
We note since $d \geq 2$ we can choose $t = \frac{\sqrt{2}}{\sqrt{3}}$ so that $\frac{1}{d t^2} \leq \frac{3}{4}$.  Thus with probability at least $1/4$ we have that 
\[|N_{in}(S_1) \setminus S_1| \in [(1 - t)d, (1 + t) d].\]
It follows that $\EE[Z] = \Omega(|S_1| d^2)$.
\end{proof}

\begin{figure}
    \centering
    \includegraphics[width=0.8\linewidth]{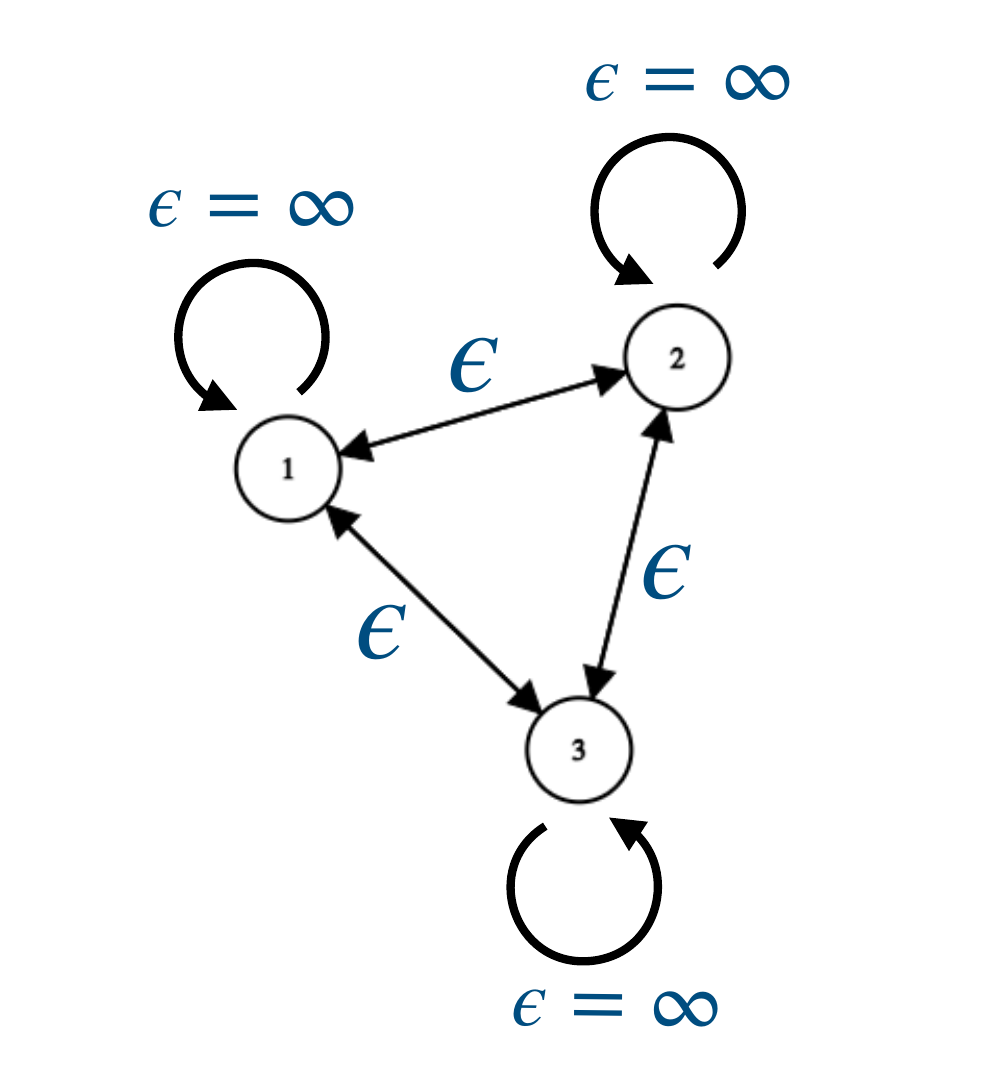}
    \caption{\textbf{Shard-Graph in \name-DP}: In \name-DP each shard when training adapters uses its own data non-privately, but enforces DP when using data from other shards. This can visualized as shard graphs with weighted edges, where the edge weight measures the privacy leakage between the interacting nodes.}
    \label{fig:safe_dp}
\end{figure}

\section{Stochastic forgetting}
Below we provide a general definition for probabilistic unlearning:

\begin{definition}
$(\alpha, \beta)-$sharded-unlearning: 
Consider an algorithm $\algo$ that, given a shard graph $G$ as input, outputs a model $\mathcal{A}(G)$ trained on the shards of $G$.  If $G'$ is a shard graph that can be obtained by removing data or nodes from $G$, we write $G' \preccurlyeq G$.  We say that $U$ is an $(\alpha, \beta)$-unlearning algorithm for $\algo$ if for all $G' \preccurlyeq G$ and events $E$
\[
\mathbb{P}(U(\algo(G), G') \in \Theta) \leq  e^{\alpha} \mathbb{P}(\algo(G') \in \Theta) + \beta.
\]
 \label{def:unlearn_graph}
\end{definition}

Note that the definition bares resemblance to the definition of differential privacy, which says that:
\begin{definition}\cite{dwork2014algorithmic}
$(\epsilon, \delta)-$differential privacy:
An algorithm $\algo$, is said to be $(\epsilon, \delta)-$DP for $\epsilon>0$ and $\delta \ll 1$, if for all adjacent datasets $D, D'$ (such that $D$ and $D'$ differ in at most one sample), and all possible events $\Theta$, the following relation holds: 
\[
\mathbb{P}(\algo(D) \in \Theta) \leq  e^{\epsilon} \mathbb{P}(\algo(D') \in \Theta) + \delta.
\]
\end{definition}

Training a model with a DP algorithm (for instance DP-SGD \cite{abadi2016deep}), enables us to perform free ($\alpha, \beta$)-unlearning for 1-forget request, by choosing $\alpha=\epsilon$ and $\beta=\delta$. ``Free" here means that the model need not be changed after a forgetting request as it still satisfies the $(\alpha, \beta)-$unlearning guarantee. However, as we get sequential forgetting requests ($\geq 1)$, the privacy bound weakens resulting in more leakage. To understand the effectiveness of differential privacy in forgetting, we need to capture its exact behaviour $(\epsilon, \delta)$ when provided with $k$ sequential forgetting requests. More precisely we get the following result for group differential privacy:

\begin{theorem}\cite{dwork2014algorithmic}
Group privacy:
Let $\algo$, be an $(\epsilon, \delta)-$DP algorithm. Then for all $D, D'$, such that $D$ and $D'$ differ in at most $k$ samples, we get the following result: 
\[
\mathbb{P}(\algo(D) \in \Theta) \leq  e^{\epsilon_g} \mathbb{P}(\algo(D') \in \Theta) + \delta_g, 
\] where $\epsilon_g=k\epsilon$ and $\delta_g=\dfrac{e^{k
\epsilon}-1}{e^{\epsilon}-1}\delta$
\label{thm:groupdp}
\end{theorem}

We observe that the group privacy result, i.e. $(\epsilon_g, \delta_g)=\big(e^{k\epsilon}, \dfrac{e^{k
\epsilon}-1}{e^{\epsilon}-1}\delta\big)$, weakens with increasing $k$. This necessitates the notion of a privacy budget $(\alpha_b, \delta_b)$ which is an upper bound for the privacy leakage accumulation with successive forgetting requests. This budget can be chosen by the user, or in some cases, constrained by the fact that $\dfrac{e^{k
\epsilon}-1}{e^{\epsilon}-1}\delta \leq 1$.

Differential privacy can detrimentally reduce the utility (accuracy) of a model. Hence, blindly using DP-SGD (or any another DP method) during training may result in models that, while strongly private (and hence not requiring frequent re-training) attain significantly lower accuracy. \name provides a simple yet effective unlearning mechanism which enables forgetting shards, by re-training the neighbouring contaminated shards (requires frequent yet reduced re-training, but higher accuracy). These observations inspire us to design an algorithm which combines \name and Stochastic Forgetting:

\textbf{\name-DP}:
For each shard $S_i$ in the shard graph $G$, we train a binary classification model which treats the data at node $S_i$ as the positive class (without privacy), and data at neighbouring nodes pointing to it as the negative classes (with privacy). When asked to forget a sample in $S_i$ we can simply drop the classifier for the shard $S_i$ (later re-train if mandated by the privacy cost of the budget), without having to worry about re-training models corresponding to other shards which used $S_i$ as a negative class (since the data of the negative shards was trained with DP).

Unlike the standard version of \name, where all the connections of shard $S_i$ (using $S_i$ for training) need to re-trained upon a single forget request, \name-DP allows us to keep the remaining neighboring shards while accounting for some cost to the privacy budget. By allowing each shard $S_i$ to use its own data non-privately, \name-DP also provides better utility compared to completely private models (trained with DP). While training with \name-DP, we compute the total $(\alpha, \beta)$ using the composition property of DP \cite{dwork2014algorithmic,mironov2017renyi,gopi2021numerical}. 

\textbf{Privacy Accounting:} The \name-DP algorithm is defined with an accompanying privacy level given by $(\epsilon, \delta)$, which sets the DP training parameters in \name-DP (using DP-SGD \cite{abadi2016deep,golatkar2022mixed}).  While the $(\alpha, \beta)$ parameters in the unlearning definition (\cref{def:unlearn_graph}) are dependent on the number of sequential forgetting requests $k$, \name-DP provides the following unlearning guarantees: $\alpha=k\epsilon$ and $\beta=\dfrac{e^{k
\epsilon}-1}{e^{\epsilon}-1}\delta$, where $\beta \leq 1$ (from \cref{thm:groupdp}).

Before the start of training, the user chooses a desired forgetting budget $(\alpha_b, \beta_b)$ (with $\beta_b \leq 1$), which measures the information contained (privacy leakage) about the user after forgetting her samples. The forgetting algorithm should respect the budget while providing non-vacuous guarantees. This provides us with the following bound for the number of unlearning request before full re-training: When $\delta \ll 1$ and $\delta \ll \beta_b$, we can approximate this result with 
\[k = \text{min}\Bigg(\dfrac{\alpha_b}{\epsilon}, \dfrac{\log\Big(\beta_b (e^{\epsilon}-1)/\delta+1\Big)}{\epsilon}\Bigg).\]
For our experiments we assume $\alpha_b=30$ and $\beta_b=1$.  We vary $\epsilon=\{1,2,3,4\}$ and $\delta={1\mathrm{e}{-10},1\mathrm{e}{-11},1\mathrm{e}{-12},1\mathrm{e}{-13}}$ and choose different values for $k$. Note that we can choose the values for $k$ by varying both $\epsilon$ and $\delta$ (ensuring $\delta \ll 1$) and choosing the best value by comparing the accuracy on a held out validation set.

\section{Additional Details of Architecture and Training}
\paragraph{Optimized graph structures.}
For our \name method, in \cref{sec:bilevel-sharding} we described the bi-level sharding structure where we split the dataset into a number $n_c$ ``coarse shards" via class balanced sub-sampling which are further split into $n_f$ ``fine shards"  via class partitioning.  The total number of shards is given by $n = n_c \cdot n_f$.  In \cref{tab:optimum_graph_structure} we report the choice of $(n_c, n_f)$ used for each dataset for each value of $n$ for the \name method.

\paragraph{Optimization.}
Whenever training InCA adapters, we train using AdamW for 30 epochs using cosine annealing and starting learning rate $\text{lr} = 0.05$; we use weight decay of $10^{-4}$.  Whenever training the linear model for the experiment in \cref{fig:inca_vs_linear} we lower the learning rate to $\text{lr}=3\mathrm{e}{-4}$ as we observed this increased performance.  For the linear model all other hyperparameters remain the same as when training the InCA adapter.

\paragraph{Architecture.}
When using InCA adapters, LayerNorm is applied to both the inputs and queries separately before they are passed through the cross-attention block.  A second LayerNorm is applied after the cross-attention before the final logits are computed as customary in ViTs.  While in general InCA adapters can be applied to any layer in the network \cite{inca}, in our experiments we always attach it to the penultimate layer, namely the end of block 22 (the input to \texttt{blocks.23.norm1}).  When using linear adapters for the experiment depicted in \cref{fig:inca_vs_linear}, we also apply LayerNorm before the fully connected layer.

\begin{figure}
    \centering
    \includegraphics[width=0.8\linewidth]{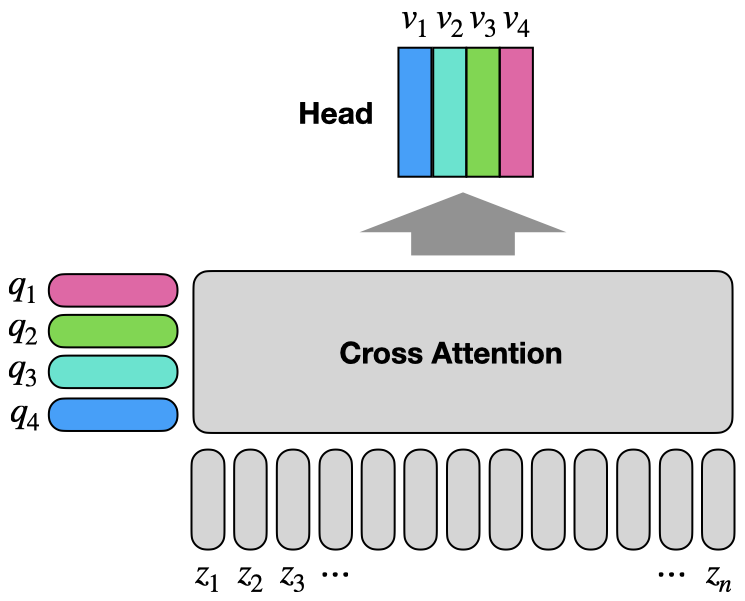}
    \caption{\textbf{The InCA Architecture.} Each class $i$ has a query $q_i$ and a head vector $v_i$.  The logit for class $i$ given by the inner product $v_i \cdot \operatorname{cross-attention}(\mathbf{z}, q_i)$ where $\mathbf{z} = (z_1, \ldots, z_n) = f_w(x)$ is the embedding extracted from a frozen pre-trained model.}
    \label{fig:open_inca}
\end{figure}

\paragraph{Prototype model.}
We note that the prototype model is only added at inference time, not training time, as adding the prototype model during training would expose each adapter to information from all other training samples, thus breaking the information compartmentalization.  When computing the prototypes $p_k$, we normalize the feature embeddings
\[ p_k = \frac{1}{N_c} \sum_{(x, y) \in D^{(k)}} \frac{f_w(x)}{\norm{f_w(x)}} \]
which makes the prototypes less sensitive to outliers and leads to better predictions.

\paragraph{Dataset details.}
In \cref{tab:suppl_datasets} we report the size of the training and testing splits and number of classes for the 7 datasets we consider, and links to access the data.
\begin{table*}[!t]
\resizebox{\textwidth}{!}{
\begin{tabular}{lcccl}
\toprule
     {\bf Dataset} & {\bf Training Images} & {\bf Testing Images} & {\bf \# Classes} & {\bf URL} \\
\hline
    Caltech-256 \cite{griffin_holub_perona_2022} & 15,418 & 15,189 & 257 & \footnotesize{\url{https://authors.library.caltech.edu/7694/}}\\
    CIFAR-100 \cite{Krizhevsky09learningmultiple} & 50,000 & 10,000 & 100 & \footnotesize{\url{https://www.cs.toronto.edu/~kriz/cifar.html}} \\
    CUB-200~\cite{WahCUB_200_2011} & 5,994 & 5,794 & 200 & \footnotesize{\url{https://www.vision.caltech.edu/datasets/cub_200_2011/}} \\
    DTD~\cite{cimpoi14describing} & 4,230 & 1,410 & 47 & \footnotesize{\url{https://www.robots.ox.ac.uk/~vgg/data/dtd/}} \\
    MIT-67~\cite{mit67recognizing} & 5,360 & 1,340 & 67 & \footnotesize{\url{https://web.mit.edu/torralba/www/indoor.html}}\\
    Stanford Cars~\cite{KrauseStarkDengFei-Fei_3DRR2013} & 8,144 & 8,041 & 196 & \footnotesize{\url{https://ai.stanford.edu/~jkrause/cars/car_dataset.html}}\\
    Stanford Dogs~\cite{KhoslaYaoJayadevaprakashFeiFei_FGVC2011} & 12,000 & 8,580 & 120 & \footnotesize{\url{http://vision.stanford.edu/aditya86/ImageNetDogs/}}\\
\bottomrule
\end{tabular}}
\caption{\textbf{Dataset Information.} We report the number of classes as well as the number of training and testing images for each dataset, as well as links to download the datasets.}
\label{tab:suppl_datasets}
\end{table*}

\section{Additional Related Work}
In our work we present \name as a method that flexibly learns a model composed of model parts that are learned using different data (as prescribed by the Shard Graph). This heterogeneous routing of data during training enables training the different model components in parallel and permits training hundreds of models quickly and efficiently. While we apply this approach for the problem of forgetting, recently ``heterogeneous data routing'' has also been the focus of much work in enabling better massive model scaling. In these works, ``heterogeneous data routing'' is used to route different data and activations to different model parts and distributing the computation to more computing nodes. Through distribution of computation, one can reduce the inference and training costs of foundation models and enable even more parameters than what is permissible by a single monolithic model \cite{Fedus2021SwitchTS}. In the work of \cite{Fedus2021SwitchTS} a large language model based on the transformer architecture is built with dynamic execution layers, where the model's intermediate activations are routed into disjoint layers based on their representations via ``switching layers''. This is generalized in the work of Pathways \cite{Barham2022PathwaysAD} that creates the necessary infrastructure to train such models on distributed computing systems and allows training state-of-the-art language models \cite{palm}. The latest developments \cite{gesmundo2023multipath,gesmundo2022munet} of this method apply heterogeneous propagation of data in connection with multi-task learning where ``agent networks'' cooperate with partner-agent representations to adapt and solve new tasks.  We note while both \cite{Fedus2021SwitchTS,Barham2022PathwaysAD} and our work utilize heterogeneous data routing, in our work data routing and compartmentalization is deterministic and is based on pre-specified data usage rules codified by the Shard Graph, as opposed to selecting the data routing based on the data's representations. Overall our work focuses on the problem of forgetting rather than multi-task learning and increasing the model scale.

\begin{table*}[t]
    \centering
 \resizebox{0.8\textwidth}{!}{
   \begin{tabular}{c|cccccccc}
 \toprule
    Dataset\textbackslash \ Num. Shards & 2 & 4 & 8 &   16 &   32 &   64 &  128 &  256  \\
 \midrule
 Caltech-256 & (2, 1) &  (4, 1) &  (4, 2) &  (4, 4) & (4, 8) & (4, 16) & (4, 32) & (8, 32) \\
   CIFAR-100 & (2, 1) & (4, 1) & (8, 1) & (8, 2) & (8, 4) & (16, 4) & (16, 8) & (16, 16)
  \\
     CUB-200 & (2, 1) & (4, 1) & (4, 2) & (8, 2) & (8, 4) & (8, 8) & (4, 32) & (8, 32) \\
         DTD & (2, 1) & (4, 1) & (4, 2) & (8, 2) & (8, 4) & (8, 8) & (16, 8) & (16, 16) \\
      MIT-67 & (2, 1) & (2, 2) & (2, 4) & (8, 2) & (8, 4) & (8, 8) & (8, 16) & (16, 16) \\
 Stanf. Cars &  (2, 1) & (2, 2) & (2, 4) & (4, 4) & (4, 8) & (8, 8) & (8, 16) & (8, 32) \\
 Stanf. Dogs & (2, 1) & (4, 1) & (8, 1) & (8, 2) & (16, 2) & (16, 4) & (16, 8) & (16, 16) \\
 \bottomrule
 \end{tabular}
}
\caption{\textbf{Coarse vs. fine shard split.} We report the number of coarse and fine shards, ($n_c$, $n_f$), used for each dataset at the different sharding levels.}
\label{tab:optimum_graph_structure}
\end{table*}

\begin{table*}[t]
    \centering
 \resizebox{0.9\textwidth}{!}{
   \begin{tabular}{c|cc|cccccccc}
 \toprule
    Dataset & No sharding & Prototypes & 2 & 4 & 8 &   16 &   32 &   64 &  128 &  256  \\
 \midrule
 Caltech-256 & 94.3\% & 93.2\% & 94.2\% & 94.1\% & 94.0\% & 93.7\% & 93.7\% & 93.5\% & 93.2\% & 93.3\% \\
   CIFAR-100 & 83.1\% & 71.2\% & 84.4\% & 84.6\% & 84.1\% & 83.3\% & 82.8\% & 82.2\% & 81.5\% & 80.9\%
  \\
     CUB-200 & 88.3\% & 85.8\% & 88.6\% & 87.9\% & 86.1\% & 85.8\% & 85.3\% & 83.5\% & 82.5\% & 84.5\% \\
         DTD & 77.8\% & 73.8\% & 78.3\% & 78.3\% & 77.1\% & 75.4\% & 75.5\% & 75.1\% & 73.9\% & 74.2\% \\
      MIT-67 & 87.9\% & 85.8\% & 88.1\% & 87.7\% & 86.9\% & 86.3\% & 86.5\% & 86.0\% & 86.2\% & 86.4\% \\
 Stanf. Cars &  75.7\% & 41.0\% & 72.6\% & 68.5\% & 62.1\% & 58.3\% & 53.2\% & 47.4\% & 41.9\% & 36.2\% \\
 Stanf. Dogs & 87.9\% & 88.0\% & 88.9\% & 89.2\% & 89.2\% & 89.0\% & 88.6\% & 88.3\% & 87.7\% & 87.8\% \\
 \bottomrule
 Avg. & 85.0\% & 77.0\% & 85.0\% & 84.3\% & 82.8\% & 81.7\% & 80.8\% & 79.4\% & 78.1\% & 77.6\%
 \end{tabular}
}
\caption{\textbf{\name Accuracy at different sharding scales.} We report the accuracy of \name across different sharding scales.}
\label{tab:safe_shard_nums}
\end{table*}

\begin{table*}[t]
    \centering
 \resizebox{0.9\textwidth}{!}{
   \begin{tabular}{c|cc|cccccccc}
 \toprule
    Dataset & No sharding & Prototypes & 2 & 4 & 8 &   16 &   32 &   64 &  128 &  256  \\
 \midrule
 Caltech-256 & 94.3\% & 93.2\% & 94.2\% & 94.1\% & 93.6\% & 92.8\% & 90.1\% & 86.7\% & 84.9\% & 84.6\% \\
   CIFAR-100 & 83.1\% & 71.2\% & 84.5\% & 84.6\% & 84.0\% & 83.2\% & 82.0\% & 80.7\% & 79.7\% & 77.6\%
  \\
     CUB-200 & 88.3\% & 85.8\% & 88.5\% & 87.6\% & 83.9\% & 73.1\% & 65.3\% & 65.1\% & 62.8\% & 63.7\% \\
         DTD & 77.8\% & 73.8\% & 78.3\% & 78.1\% & 76.1\% & 74.6\% & 71.8\% & 65.1\% & 58.6\% & 52.6\% \\
      MIT-67 & 87.9\% & 85.8\% & 88.1\% & 87.7\% & 86.8\% & 85.2\% & 83.6\% & 81.5\% & 77.5\% & 77.9\% \\
 Stanf. Cars &  75.7\% & 41.0\% & 72.1\% & 66.0\% & 55.8\% & 40.9\% & 23.7\% & 17.1\% & 13.2\% & 7.0\% \\
 Stanf. Dogs & 87.9\% & 88.0\% & 88.8\% & 89.2\% & 89.2\% & 88.6\% & 86.5\% & 83.1\% & 81.1\% & 79.8\%  \\
 \bottomrule
 Avg. & 85.0\% & 77.0\% & 84.9\% & 83.9\% & 81.4\% & 76.9\% & 71.9\% & 68.5\% & 65.4\% & 63.3\%
 \end{tabular}
}
\caption{\textbf{SISA Accuracy at different sharding scales.} We report the accuracy of SISA across different sharding scales.}
\label{tab:sisa_shard_nums}
\end{table*}

\begin{table*}[t]
    \centering
 \resizebox{0.9\textwidth}{!}{
   \begin{tabular}{c|cc|cccccccc}
 \toprule
    Dataset & No sharding & Prototypes & 2 & 4 & 8 &   16 &   32 &   64 &  128 &  256  \\
 \midrule
 Caltech-256 & 94.3\% & 93.2\% & 94.2\% & 94.1\% & 93.6\% & 92.8\% & 90.1\% & 86.8\% & 86.4\% & 90.5\% \\
   CIFAR-100 & 83.1\% & 71.2\% & 84.5\% & 84.6\% & 84.0\% & 83.2\% & 82.0\% & 80.7\% & 79.7\% & 77.7\% \\
     CUB-200 & 88.3\% & 85.8\% & 88.5\% & 87.6\% & 83.9\% & 73.1\% & 65.7\% & 67.8\% & 73.5\% & 83.5\% \\
         DTD & 77.8\% & 73.8\% & 78.3\% & 78.1\% & 76.1\% & 74.6\% & 71.8\% & 66.7\% & 67.5\% & 71.2\% \\
      MIT-67 & 87.9\% & 85.8\% & 88.1\% & 87.7\% & 86.9\% & 85.2\% & 83.9\% & 81.7\% & 81.2\% & 84.4\% \\
 Stanf. Cars &  75.7\% & 41.0\% & 72.1\% & 66.0\% & 55.8\% & 40.9\% & 23.9\% & 17.8\% & 16.9\% & 25.5\% \\
 Stanf. Dogs & 87.9\% & 88.0\% & 88.8\% & 89.2\% & 89.2\% & 88.6\% & 86.6\% & 83.2\% & 81.9\% & 83.5\% \\
 \bottomrule
 Avg. & 85.0\% & 77.0\% & 84.9\% & 83.9\% & 81.4\% & 76.9\% & 72.0\% & 69.3\% & 69.6\% & 73.8\%
 \end{tabular}
}
\caption{\textbf{ProtoSISA Accuracy at different sharding scales.} We report the accuracy of ProtoSISA across different sharding scales.}
\label{tab:protosisa_shard_nums}
\end{table*}

\clearpage

\end{document}